\documentclass[11pt]{article}

\usepackage{fullpage,times,url}

\usepackage{amsthm,amsfonts,amsmath,amssymb,epsfig,color,float,graphicx,verbatim}
\usepackage{algorithm,algorithmic,natbib}
\usepackage{mathtools}
\usepackage{multicol}

\usepackage{hyperref}
\hypersetup{
	colorlinks   = true, 
	urlcolor     = blue, 
	linkcolor    = blue, 
	citecolor   = black 
}

\newtheorem{theorem}{Theorem}

\newtheorem{lemma}{Lemma}
\newtheorem{corollary}{Corollary}

\newtheorem{remark}{Remark}

\newcommand{\relu}[1]{\left[ #1 \right]_+}
\newcommand{\set}[1]{\left\lbrace#1\right\rbrace}
\newcommand{\p}[1]{\left( #1 \right)}
\newcommand{\pcc}[1]{\left[ #1 \right]}

\newcommand{\direction}{\bu}
\newcommand{\weights}{\bw_1^n}
\newcommand{\weightstilde}[1]{\tilde{\bw}_{1,#1}^n}
\newcommand{\weightsprime}{\bw_1^{\prime n}}
\newcommand{\weightsstar}{\bw_1^{* n}}
\newcommand{\spnorm}[1]{\left|\left|#1\right|\right|_{\text{sp}}}
\newcommand{\reals}{\mathbb{R}}
\newcommand{\E}{\mathbb{E}}

\newcommand{\abs}[1]{\left| #1 \right|}

\newcommand{\bI}{\mathbf{I}}

\newcommand{\be}{\mathbf{e}}
\newcommand{\bx}{\mathbf{x}}
\newcommand{\bw}{\mathbf{w}}

\newcommand{\bu}{\mathbf{u}}
\newcommand{\bv}{\mathbf{v}}

\newcommand{\bn}{\mathbf{n}}

\newcommand{\Ncal}{\mathcal{N}}

\newcommand{\norm}[1]{\left|\left|#1\right|\right|}

\newtheorem{example}{Example}

\newcommand{\secref}[1]{Sec.~\ref{#1}}
\newcommand{\subsecref}[1]{Subsection~\ref{#1}}
\newcommand{\figref}[1]{Fig.~\ref{#1}}
\renewcommand{\eqref}[1]{Eq.~(\ref{#1})}
\newcommand{\lemref}[1]{Lemma~\ref{#1}}

\newcommand{\thmref}[1]{Thm.~\ref{#1}}

\title{Spurious Local Minima are Common\\ in Two-Layer ReLU Neural Networks}
\author{Itay Safran\\Weizmann Institute of 
Science\\\texttt{itay.safran@weizmann.ac.il}\and
	Ohad Shamir
	\\Weizmann Institute of Science\\{\texttt{ohad.shamir@weizmann.ac.il}}
}
\date{}

\begin{document}

\maketitle
\begin{abstract}
	We consider the optimization problem associated with training simple ReLU 
	neural networks of the form $\bx\mapsto \sum_{i=1}^{k}\max\{0,\bw_i^\top 
	\bx\}$ with respect to the squared loss. We provide a computer-assisted 
	proof that even if the input distribution is standard Gaussian, even if the 
	dimension is arbitrarily large, and even if the target values are generated by 
	such a network, with orthonormal parameter vectors, the problem can 
	still have spurious local minima once $6\le k\le 20$. By a concentration of measure argument, 
	this implies that in high input dimensions, \emph{nearly all} target networks of 
	the relevant sizes lead to spurious local minima. Moreover, 
	we conduct experiments which show that the probability of hitting such 
	local minima is quite high, and increasing with the network size. On the 
	positive side, mild over-parameterization appears to drastically reduce 
	such local minima, indicating that an over-parameterization assumption is 
	necessary to get a positive result in this setting. 
\end{abstract}

\section{Introduction}

One of the biggest mysteries of deep learning is why neural networks are 
successfully trained in practice using gradient-based methods, despite the 
inherent non-convexity of the associated optimization problem. For example, 
non-convex problems can have poor local minima, which will cause any local 
search method (and in particular, gradient-based ones) to fail. 
Thus, it is natural to ask what types of assumptions, in the context of 
training neural networks, might mitigate such problems. For example, recent 
work has shown that other non-convex learning problems, such as phase 
retrieval, matrix completion, dictionary learning, and tensor decomposition, do 
not have spurious local minima under suitable assumptions, in which case local 
search methods have a chance of succeeding 
(e.g., 
\citep{ge2015escaping,sun2015nonconvex,ge2016matrix,bhojanapalli2016global}). 
Is it possible to prove similar positive results for neural networks?

In this paper, we focus on perhaps the simplest non-trivial ReLU neural 
networks, namely predictors of the form
\[
\bx\mapsto \sum_{i=1}^{k}[\bw_i^\top 
\bx]_+
\]
for some $k> 1$, where $[z]_+=\max\{0,z\}$ is the ReLU function, $\bx$ is a 
vector in $\reals^d$, and $\bw_1,\ldots,\bw_k$ are parameter 
vectors. We consider directly optimizing the expected squared loss, where the 
input is standard Gaussian, and in the realizable case -- namely, that the 
target values are generated 
by a network of a similar architecture:
\begin{equation}\label{eq:objfun}
\min_{\bw_1,\ldots,\bw_k}\E_{\bx\sim 
\Ncal(\mathbf{0},I)}\left[\frac{1}{2}\left(\sum_{i=1}^{k}[\bw_i^\top\bx]_+
-\sum_{i=1}^{k}[\bv_i^\top\bx]_+\right)^2~\right]~.
\end{equation}
Note that here, the choice $\bw_i=\bv_{\sigma(i)}$ (for all $i=1,\ldots,k$ and 
any permutation $\sigma$) is a global minimum with zero expected loss. Several 
recent papers analyzed such objectives, in the hope of showing that it does not 
suffer from spurious local minima (see related work below for more 
details). 

Our main contribution is to prove that unfortunately, this conjecture is false, 
and that \eqref{eq:objfun} indeed has spurious local minima once $6\le k\le20 
$. Moreover, this is true even if the dimension is unrestricted, and even if 
we assume that $\bv_1,\ldots,\bv_k$ are orthonormal 
vectors. In fact, since in high dimensions randomly-chosen vectors are 
approximately orthogonal, and the landscape of the objective function is robust 
to small perturbations, we can show that spurious local minima 
exist for \emph{nearly all} neural network problems as in \eqref{eq:objfun}, in 
high enough dimension 
(with respect to, say, a Gaussian distribution over $\bv_1,\ldots,\bv_k$). 
Moreover, we show experimentally that these 
local minima are not pathological, and that standard gradient descent can 
easily get trapped in them, with a probability which seems to increase towards 
$1$ with the network size.

Our proof technique is a bit unorthodox. Although it is possible to write down 
the gradient of \eqref{eq:objfun} in closed form (without the expectation), 
it is not clear how to get analytical expressions for its 
roots, and hence characterize the stationary points of \eqref{eq:objfun}. As 
far as we know, an analytical expression for the roots might not even exist. 
Instead, we employed the following strategy: We ran standard gradient descent 
with random initialization on the objective function, until we reached a point 
which is both suboptimal (function value being significantly higher than $0$); 
approximate stationary (gradient norm very close to $0$); and with a strictly 
positive definite Hessian (with minimal eigenvalue significantly larger than 
$0$). We use a computer to verify these conditions in a formal manner,  
avoiding floating-point arithmetic and the possibility of 
rounding errors. Relying on these numbers, we employ a Taylor expansion 
argument, to show that we must have arrived at a point very close to a local 
(non-global) minimum of \eqref{eq:objfun}, hence establishing the existence of 
such minima. 

On the more positive side, we show that an additional 
\emph{over-parameterization} 
assumption appears to be very effective in mitigating these local minima 
issues: Namely, we use a network larger than that needed with unbounded computational power, and replace \eqref{eq:objfun} with
\begin{equation}\label{eq:objfunover}
\min_{\bw_1,\ldots,\bw_k}\E_{\bx\sim 
	\Ncal(\mathbf{0},I)}\left[\frac{1}{2}\left(\sum_{i=1}^{n}[\bw_i^\top\bx]_+
-\sum_{i=1}^{k}[\bv_i^\top\bx]_+\right)^2~\right]~,
\end{equation}
where $n>k$. In our experiments with $k,n$ up to size $20$, we observe that 
whereas $n=k$ leads to plenty of local minima, $n=k+1$ leads to much fewer 
local minima, whereas no local minima were encountered once $n\geq k+2$ 
(although those might still 
exist for larger values of $k,n$ than those we tried). Thus, although 
\eqref{eq:objfun} has local minima, we conjecture that \eqref{eq:objfunover} 
might still be proven to have no bad local minima, but this would 
\emph{necessarily} require $n$ to be sufficiently larger than $k$. 

The paper is structured as follows: After surveying related work below, we 
provide our main results and proof ideas in \secref{sec:main}. 
\secref{sec:experiments} provide additional experimental details about the 
local minima found, as well empirical evidence about the likelihood of reaching 
them using gradient descent. Detailed proofs are in \secref{sec:proofs}. 

\subsection{Related Work}

There is a large and rapidly increasing literature on the optimization theory  
of neural networks, surveying all of which is well outside our scope. Thus, in 
this subsection, we only briefly survey the works most relevant to ours.

We begin by noting that when minimizing the average loss over some arbitrary 
finite dataset, it is easy to construct problems where even for a single 
neuron ($k=1$ in \eqref{eq:objfun}), there are many spurious local minima 
(e.g., \citep{auer1996exponentially,swirszcz2016local}). Moreover, the 
probability of starting at a basin of such local minima is 
exponentially high in the dimension \citep{safran2016quality}. On the other 
hand, it is known that if the network is over-parameterized, and large enough 
compared to the data size, then there are no local minima 
\citep{poston1991local,livni2014computational,haeffele2015global,zhang2016understanding,soudry2016no,soltanolkotabi2017theoretical,nguyen2017loss,boob2017theoretical}.
In any case, neither these positive nor negative results apply here, as we are 
interested in the expected (population) loss 
with respect to the Gaussian distribution, which is of course non-discrete. 
Also, several recent works have studied learning neural networks under a 
Gaussian distribution assumption 
(e.g., 
\citet{janzamin2015beating,brutzkus2017globally,du2017gradient,li2017convergence,feizi2017porcupine,zhang2017electron,ge2017learning}),
but using a network architecture different than ours, or focusing on 
 algorithms rather than the geometry of the optimization problem. Finally, 
 \citet{shamir2016distribution} provide hardness results for training neural 
 networks even under distributional assumptions, but these do not apply when 
 making strong assumptions on \emph{both} the input distribution and the 
 network generating the data, as we do here.

For \eqref{eq:objfun}, a few works have shown that there are no spurious 
local minima, or that gradient descent will succeed in reaching a global 
minimum, provided the $\bv_i$ vectors are in general position or orthogonal 
\citep{zhong2017recovery,soltanolkotabi2017theoretical,tian2017analytical}. 
However, these results either apply only to $k=1$, assume the algorithm is 
initialized close to a global optimum, or analyze the geometry of the problem 
only on some restricted subset of the parameter space.

The empirical observation that gradient-based methods may not work well on 
\eqref{eq:objfun} has been made in \citet{livni2014computational}, and more 
recently in \citet{ge2017learning}. Moreover,  \citet{livni2014computational} 
empirically observed that over-parameterization seems to help. However, our 
focus here is to \emph{prove} the existence of such local minima, as well as 
more precisely quantify their behavior as a function of the network sizes.

\section{Main Result and Proof Technique}\label{sec:main}

Before we begin, a small note on terminology: When referring to local minima of 
a function $F$ on Euclidean space, we always mean spurious local minima 
(i.e., points $\bw$ such that $\inf_{\bw}F(\bw)<F(\bw)\leq F(\bw')$ for all 
$\bw'$ in some open neighborhood of $\bw$). 

Our basic result is the following:
\begin{theorem}\label{thm:main}
	Consider the optimization problem 
	\[
		\min_{\bw_1,\ldots,\bw_n\in \reals^{k}} \E_{\bx\sim 
		\Ncal(\mathbf{0},I)}\left[\frac{1}{2}\left(\sum_{i=1}^{n}[\bw_i^\top\bx]_+
		-\sum_{i=1}^{k}[\bv_i^\top\bx]_+\right)^2~\right]~,	
	\]
	where $\bv_1,\ldots,\bv_{k}$ are orthogonal unit vectors in $\reals^{k}$. 
	Then for $n=k\in\{6,7,\ldots,20\}$, as well 
	as $(k,n)\in\{\p{8,9}, \p{10,11}, \p{11,12}, \dots, \p{19,20}\}$, the objective 
	function above has spurious local minima.
\end{theorem}

\begin{remark} For $k,n$ smaller than $6$, we were unable to find local minima 
using our proof technique, since gradient descent always seemed to converge to 
a global minimum. Also, although we have verified the theorem only up 
to $k,n\leq 20$, the result strongly suggests that there are local minima for 
larger values as well. See \secref{sec:experiments} for some examples of the 
local minima found. 
\end{remark}

The theorem assumes a fixed input dimension, and a particular choice of 
$\bv_1,\ldots,\bv_k$. However, these assumptions are not necessary and can be 
relaxed, as demonstrated by the following corollary:

\begin{corollary}\label{cor:main}
	\thmref{thm:main} also applies if the space $\reals^{k}$ is replaced by 
	$\reals^{d}$ for any $d>k$ (with $\bx$ distributed as a standard Gaussian 
	in that space). Moreover, if $\bv_1,\ldots,\bv_k$ are 
	chosen i.i.d. from a Gaussian distribution $\Ncal(\mathbf{0},cI)$ (for any 
	$c>0$), the theorem still 
	holds with probability at least $1-\exp(-\Omega(d))$. 
\end{corollary}

\begin{remark}
The corollary is not specific to a Gaussian distribution over 
$\bv_1,\ldots,\bv_k$, and can be generalized to any distribution for which 
$\bv_1,\ldots,\bv_k$ are approximately orthogonal and of the same norm in high 
dimensions (see below for details).
\end{remark}

We now turn to explain how these results are derived, starting with 
\thmref{thm:main}. In what follows, we let 
$\weights=(\bw_1,\ldots,\bw_n)\in \reals^{kn}$ be the vector of parameters, and let 
$F(\weights)$ be the 
objective function defined in \thmref{thm:main} (assuming $k,n$ are fixed). 
We will also assume that $F$ is thrice-differentiable in a neighborhood of $ \weights $ (which will be shown to be true as part of our proofs), with a gradient $\nabla F(\cdot)$ and a Hessian $\nabla^2 F(\cdot)$.

Clearly, a global minimum of $F$ is obtained by $\bw_i=\bv_i$ for all 
$i=1,\ldots,k$ (and $\bw_i=\mathbf{0}$ otherwise), in which case $F$ attains a 
global minimum of $0$. Thus, to prove \thmref{thm:main}, it 
is sufficient to find a point $\weights\in\reals^{kn}$ such that $\nabla F(\weights)=0$,  
$\nabla^2 F(\weights)\succeq 0$, and $F(\weights)>0$. The major difficulty is showing the 
existence of points where the first condition is fulfilled: Gradient descent 
allows us to find points where $\nabla F(\weights)\approx 0$, but it is very 
unlikely to return a point where $\nabla F(\weights)=0$ exactly. Instead, we use a 
Taylor-expansion argument (detailed below), to show that if we found a point 
such that $\nabla F(\weights)$ is sufficiently close to $0$, as well as $\nabla^2 
F(\weights)\succ 0$ and $F(\weights)>0$, then $\weights$ must 
be close to a local minimum. 

The second-order Taylor expansion of a multivariate, thrice-differentiable 
function $ F $ about a 
point $\weights\in \reals^{kn}$, in a direction given by a unit vector $\bu\in 
\reals^{kn}$ and using a Lagrange remainder term, is given by
\begin{align*}
	F\p{\weights+t\bu}=F\p{\weights} +& t\sum_{i_1}\frac{\partial}{\partial 
	\bw_{1,i_1}^n}F\p{\weights}u_{i_1}+ 
	\frac{1}{2}t^2\sum_{i_1,i_2}\frac{\partial^2}{\partial \bw_{1,i_1}^n\partial 
	\bw_{1,i_2}^n}F\p{\weights}u_{i_1}u_{i_2}\\ +& 
	\frac{1}{6}t^3\sum_{i_1,i_2,i_3}\frac{\partial^3}{\partial \bw_{1,i_1}^n\partial 
	\bw_{1,i_2}^n\partial \bw_{1,i_3}^n}F\p{\weights+\xi\bu}u_{i_1}u_{i_2}u_{i_3},
\end{align*}

for some $ \xi\in\p{0,t} $, and where $ \bw_{1,i}^n $ denotes the $ i $-th coordinate of $ \weights $.
Denoting the remainder term as $ R_{\weights,\bu,t} $, we have
\begin{equation}\label{eq:taylor}
	F\p{\weights+t\bu} = F\p{\weights} + t\nabla F\p{\weights}^{\top}\bu + 
	\frac{1}{2}t^2\bu^{\top}\nabla^2F\p{\weights}\bu + \frac{1}{6}t^3R_{\weights,\bu,t}~.
\end{equation}

Now, suppose that the point $\weights$ we obtain by gradient descent satisfies 
$\norm{\nabla F(\weights)}\leq \epsilon$, $\nabla^2F(\weights)\succeq \lambda_{\min}\cdot I$ and 
$|R_{\weights,\bu,t}|\leq B_t$ (for some positive $\lambda_{\min},\epsilon,B_t$), uniformly 
for all unit vectors $\bu$. Fix some $ \alpha>0 $ and let $ B=\sup_{t\in[0,\alpha]}B_t $. By the Taylor 
expansion above, this implies that for any $ t\in[0,\alpha] $ and all unit $\bu$,
\begin{align*}
F(\weights+t\bu)~&\geq~ F(\weights)-t\norm{\nabla
F(\weights)}\cdot\norm{\bu}+\frac{t^2}{2}\lambda_{\min}\norm{\bu}^2-\frac{t^3}{6}B\notag\\
&=~F(\weights)-\epsilon t+\frac{\lambda_{\min}t^2}{2}-\frac{Bt^3}{6}\notag\\
&=~F(\weights)+t\left(\frac{\lambda_{\min}}{2}t-\frac{B}{6}t^2-\epsilon\right).~
\end{align*}
An elementary calculation reveals that the term 
$t\left(\frac{\lambda_{\min}}{2}t-\frac{B}{6}t^2-\epsilon\right)$ is strictly 
positive for any 
\[
t~\in~ 
\left(\frac{3\lambda_{\min}-\sqrt{9\lambda_{\min}^2-24B\epsilon}}{2B}~,~
\frac{3\lambda_{\min}+\sqrt{9\lambda_{\min}^2-24B\epsilon}}{2B}\right)~,
\]
(and in particular, in the closed interval of 
$\frac{3\lambda_{\min}\pm\sqrt{9\lambda_{\min}^2-25B\epsilon}}{2B}$). Letting $ r\coloneqq\frac{3\lambda_{\min}-\sqrt{9\lambda_{\min}^2-25B\epsilon}}{2B} $, and assuming $ r<\alpha $, we get that there is some small closed ball $\bar{B}_r$ of radius $r$ centered 
at $\weights$ (and with boundary $S$), such that
\[
F(\weights) < \min_{\weightsprime\in S} F(\weightsprime).
\]
Moreover, since $F$ is continuous, it is minimized over $\bar{B}_r$ at some 
point $\weightsstar$. But then
\begin{equation}\label{eq:Scond}
F(\weightsstar)=\min_{\weightsprime\in \bar{B}_r}F(\weightsprime)\leq F(\weights) < \min_{\weightsprime\in S} F(\weightsprime),
\end{equation}
so $\weightsstar$ must reside in the interior of $\bar{B}_r$. Thus, it is minimal in an open 
neighborhood containing it, hence it is a local 
minimum. Overall, we have arrived at the following key lemma:

\begin{lemma}\label{lem:key}
	Assume that for some $\epsilon,B,\alpha>0$, it holds that $\norm{\nabla F\p{\weights}}\leq \epsilon$ and $$ \sup_{\substack{t\in [0,\alpha]\\ \bu:\norm{\bu}=1}} \abs{R_{\weights,\bu,t}}\leq B, $$ let $ \lambda_{\min}>0 $ denote the smallest 
	eigenvalue of $ \nabla^2F\p{\weights} $, and let
	\[
		r\coloneqq 
		\frac{3\lambda_{\min}-\sqrt{9\lambda_{\min}^2-25B\epsilon}}{2B}.
	\]
	If $ 9\lambda_{\min}^2- 25B\epsilon\ge 0 $ and $ r<\alpha $ then the function $F$ contains a local minimum, within a distance of at most	$ r $ from $ \weights $. 
\end{lemma}

The only missing element is that the local minimum might be a global minimum. 
To rule this out, one can simply use the fact that $F$ is a Lipschitz function, 
so that if $F(\weights)$ is much larger than $0$, the neighboring local minimum can't 
have a value of $0$, and hence cannot be global:
\begin{lemma}\label{lem:key2}
	Under the conditions of \lemref{lem:key}, if it also holds that
	\begin{equation}\label{eq:non_global_cond}
		F\p{\weights} > r^2\p{\frac{1}{2} +n\p{n-1} \p{\frac{\p{\max_i 
		\norm{\bw_{i}}+r}}{2\pi\p{\min_i \norm{\bw_i}-r}}+\frac{1}{2}} + 
		\frac{nk \cdot \max_i \norm{\bv_i}}{2\pi \p{\min_i \norm{\bw_i}-r}}} + 
		r\epsilon,
	\end{equation}
	then the local minimum is non-global.
\end{lemma}
The formal proof of this lemma appears in 
\subsecref{subsec:objective_lipschitzness}. 

Most of the technical proof of \thmref{thm:main} consists in rigorously 
verifying the conditions of \lemref{lem:key} and \lemref{lem:key2}. A major 
hurdle is that floating-point calculations are not guaranteed to be accurate 
(due to the 
possibility of round-off and other errors), so for a formal proof, one needs to 
use software that comes with guaranteed numerical accuracy. In our work, we 
chose to use variable precision arithmetic (VPA), a standard package of MATLAB 
which is based on symbolic arithmetic, and allows performing elementary 
numerical computations with an arbitrary 
number of guaranteed digits of precision. The 
main technical issue we faced is that some calculations are not easily done 
with a few elementary arithmetical operations (in particular, the standard way 
to compute $\lambda_{\min}$ would be via a spectral decomposition of the 
Hessian matrix). The bulk of the proof consists of showing how we bound the 
quantities relevant to \lemref{lem:key} in an elementary manner. 

Finally, we turn to discuss how Corollary \ref{cor:main} is proven, given 
\thmref{thm:main} (see \subsecref{subsec:cor_proof} for a more formal 
derivation). 
The proof idea is that the objective does not have any ``non-trivial'' 
structure outside the span of $\bv_1,\ldots,\bv_k$. Therefore, if we take a 
local minima for $\reals^{k}$, and pad it 
with $d-k$ zeros, we get a point in $\reals^d$ for which the gradient's norm is 
unchanged, the Hessian has the same spectrum for any $ d\ge k+1 $, and the third derivatives 
are still bounded. Hence, that point is a local minimum in the 
higher-dimensional problem as well. As to the second part of the corollary, the 
only property of the Gaussian distribution we need is that in high dimensions, 
if we sample $\bv_1,\ldots,\bv_k$, then we are overwhelmingly likely to get 
approximately orthogonal vectors with approximately the same norm. Hence, up to 
rotation and scaling, we get a small perturbation $\tilde{F}$ of the objective 
$F$ considered in \thmref{thm:main}. Moreover, for large enough $d$, we can 
make the perturbation arbitrarily small, uniformly in some compact domain. Now, 
recall that we prove the existence of some local minimum $\weightsstar$, by 
showing that $F(\weights)<\min_{\weightsprime\in S} F(\weightsprime)$ in some 
small sphere $S$ enclosing $\weights$. If the perturbations are small enough, 
we also have $\tilde{F}(\weights)<\min_{\weightsprime\in S} 
\tilde{F}(\weightsprime)$, which by arguments similar to 
before, imply that $\weights$ is close to a local minimum of $\tilde{F}$. 

\section{Experiments}\label{sec:experiments}

So far, our technique proves the \emph{existence} of local minima for the objective 
function in \eqref{eq:objfunover}. However, this does 
not say anything about the likelihood of gradient descent to reach them.
We now turn to study this question empirically. 

For each value of $ \p{k,n} $, where $k\in\pcc{20}$ and $n\in\set{k,\dots,20}$, 
we ran 1000 instantiations of gradient descent on the objective in 
\eqref{eq:objfunover}, each starting from a different random 
initialization\footnote{We used standard Xavier initialization: Each weight 
vector $\bw_i$ was samples i.i.d. from a Gaussian distribution in $\reals^k$, 
with zero mean and covariance $\frac{1}{k}I$.}. Each instantiation was ran with 
a fixed step size of $0.1$, until reaching a candidate stationary point / local 
minima (the stopping criterion was that the gradient norm w.r.t. any $\bw_i$ is at most $10^{-9}$). Points 
obtaining objective values less than $ 10^{-3} $ were ignored as those are 
likely to be close to a global minimum. Interestingly, no points with value 
between $ 10^{-3} $ and $ 10^{-2} $ were found. For all remaining candidate 
points, we verified that the conditions in Lemmas \ref{lem:key} and 
\ref{lem:key2} are met\footnote{
	Since running our algorithm for all suspicious points found on all 
	architectures is time consuming, we instead identified points that are 
	equivalent up to permutations on the order of neurons and of the data 
	coordinates, since the objective is invariant under such permutations. By 
	bounding the maximal Euclidean distance between these points and using the 
	Lipschitzness of the objective and its Hessian (see \thmref{thm:F_bound} 
	and \lemref{lem:hess_lipschitz}), this allowed us to run the algorithm on a 
	single representative from a family of equivalent points and speed up the 
	running time drastically. Also, the objective was tested to be 
	thrice-differentiable in all enclosing balls of radii returned by the 
	algorithm. Specifically, we ensured that no two such balls intersect (which 
	results in two identical neurons, where the objective is not 
	thrice-differentiable) and that no ball contains the origin (which results 
	in a neuron with weight $ \mathbf{0} $, where again the objective is not 
	thrice-differentiable).}
 to conclude that these points are indeed close to 
spurious local minima (in all cases, the distance turned out to be less than 
$2\cdot 10^{-6}$). Our verification process included verifying thrice-differentiability in the enclosing balls containing the minimum by asserting they contain no singular points, hence the objective is an analytical expression when restricted to these balls where differentiability follows.

In Tables \ref{tbl:n_eq_k} and \ref{tbl:n_neq_k}, we summarize the percentage 
of instantiations which were verified to converge close to a spurious local 
minimum, as a function of $ 
k,n $. We note that among candidate points found, only a tiny fraction could 
not be verified to be local minima (this only occured for network sizes $ 
\p{k,n}\in\set{\p{15,16},\p{17,18},\p{20,20}} $, and consist only $ 
0.1\%,2.4\%,0.9\% $ of the instantiations respectively). In the tables, we also 
provide the minimal eigenvalue of the Hessian of the objective, and the 
objective value (or equivalently, the optimization error) at the points found, 
averaged over the instantiations\footnote{Since all points are extremely close to a local minimum, the objective at the minimum is essentially the same, up to a deviation on order less than $1.1\cdot 10^{-9}$. Also, the minimal eigenvalues vary by at most $ 5.7\cdot10^{-4} $.}. Note that since the minimal eigenvalue is strictly positive and varies slightly inside the enclosing ball, this indicates that these are in fact strict local minima. 
As the tables demonstrate, the probability of converging to a spurious local 
minimum increases rapidly with $ k,n $, and suggests that it eventually become 
overwhelming as long as $n\approx k$. However, on a positive note, mild 
over-parameterization seems to remedy this, as no local minima were found for $ 
n\ge k+2 $ where $ n\le20 $, and local minima for $ n=k+1 $ are much more 
scarce than for $ n=k $. We leave the investigation of local minima for larger 
values of $ k,n $ to future work.

\begin{table}[t]
	\begin{minipage}{.45\linewidth}
		\centering
		\caption{Spurious local minima found for $ n=k $}
		\vskip 0.3cm
		\label{tbl:n_eq_k}
		\begin{tabular}{l|l|c|c|c}
			k & n & \% of runs & Average & Average\\&& converging to & minimal 
			& objective \\&& local minima & eigenvalue & value\\
			\hline
			6 & 6 & 0.3\% & 0.0047 & 0.025 \\
			7 & 7 & 5.5\% & 0.014 & 0.023 \\
			8 & 8 & 12.6\% & 0.021 & 0.021 \\
			9 & 9 & 21.8\% & 0.027 & 0.02 \\
			10 & 10 & 34.6\% & 0.03 & 0.022 \\
			11 & 11 & 45.5\% & 0.034 & 0.022 \\
			12 & 12 & 58.5\% & 0.035 & 0.021 \\
			13 & 13 & 73\% & 0.037 & 0.022 \\
			14 & 14 & 73.6\% & 0.038 & 0.023 \\
			15 & 15 & 80.3\% & 0.038 & 0.024 \\
			16 & 16 & 85.1\% & 0.038 & 0.027 \\
			17 & 17 & 89.7\% & 0.039 & 0.027 \\
			18 & 18 & 90\% & 0.039 & 0.029 \\
			19 & 19 & 93.4\% & 0.038 & 0.031 \\
			20 & 20 & 94\% & 0.038 & 0.033
		\end{tabular}
		\vskip 1cm
	\end{minipage}%
	\hfill~~~~\vline~\vline\hfill
	\begin{minipage}{.45\linewidth}
		\caption{Spurious local minima found for $ n\neq k $}
		\label{tbl:n_neq_k}
		\vskip 0.3cm
		\begin{tabular}{l|l|c|c|c}
			k & n & \% of runs & Average & Average\\&& converging to & minimal 
			& objective \\&& local minima & eigenvalue & value\\
			\hline
			8 & 9 & 0.1\% & 0.0059 & 0.021 \\
			10 & 11 & 0.1\% & 0.0057 & 0.018 \\
			11 & 12 & 0.1\% & 0.0056 & 0.017 \\
			12 & 13 & 0.3\% & 0.0054 & 0.016 \\
			13 & 14 & 1.5\% & 0.0015 & 0.038 \\
			14 & 15 & 5.5\% & 0.002 & 0.033 \\
			15 & 16 & 10.1\% & 0.004 & 0.032 \\
			16 & 17 & 18\% &  0.0055 & 0.031 \\
			17 & 18 & 20.9\% & 0.007 & 0.031 \\
			18 & 19 & 36.9\% & 0.0064 & 0.028 \\
			19 & 20 & 49.1\% & 0.0077 & 0.027
		\end{tabular}
	\end{minipage} 
\end{table}

In \figref{fig:obj_val_cdf}, we show the distribution of the objective values 
obtained in the points found, over the 1000 instantiations of several 
architectures. The figure clearly indicates that apart from a higher chance of 
converging to local minima, larger architectures also tend to have worse values 
attained on these minima.

Finally, in examples \ref{ex:k6n6} and \ref{ex:k8n9} below, we present some 
specific local minima found for $n=k=6$ 
and $k=8,n=9$, and discuss their properties. We note that these are the 
smallest networks (with $n=k$ and $n\neq k$ respectively) for which we were 
able to find such points. 

\begin{example}\label{ex:k6n6}
	Out of 1000 gradient descent instantiations for $ n=k=6 $, three converged 
	close to a local minimum. All three were verified to be 
	essentially identical (after permuting the neurons and up to an Euclidean 
	distance of $1.2\cdot 10^{-8}$), and have the following form:
	\begin{equation*}
		\bw_1^6=\begin{bmatrix}
			-0.6015 & 0.3080 & 0.3080 & 0.3080 & 0.3080 & 0.3080 \\
			0.2245 & 0.9867 & -0.0504 & -0.0504 & -0.0504 & -0.0504 \\
			0.2245 & -0.0504 & 0.9867 & -0.0504 & -0.0504 & -0.0504 \\
			0.2245 & -0.0504 & -0.0504 & 0.9867 & -0.0504 & -0.0504\\
			0.2245 & -0.0504 & -0.0504 & -0.0504 & 0.9867 & -0.0504 \\
			0.2245 & -0.0504 & -0.0504 & -0.0504 & -0.0504 & 0.9867
		\end{bmatrix},
	\end{equation*}
	where the parameter vector of each of the $ 6 $ neurons corresponds to a 
	column of $ \bw_1^6 $.
	The Hessian of the objective at $ \bw_1^6 $, $ \nabla^2F\p{\bw_1^6} 
	$, was confirmed to have minimal eigenvalue $ 
	\lambda_{\min}\p{\nabla^2F\p{\bw_1^6}} \ge 0.004699 $. This implied 
	that all three suspicious points found for $ n=k=6 $ are of distance at 
	most $ r=1.12\cdot10^{-7} $ from a local minimum with objective value at 
	least $ 0.02508 $.
\end{example}

\begin{figure}
	\includegraphics[scale=.35]{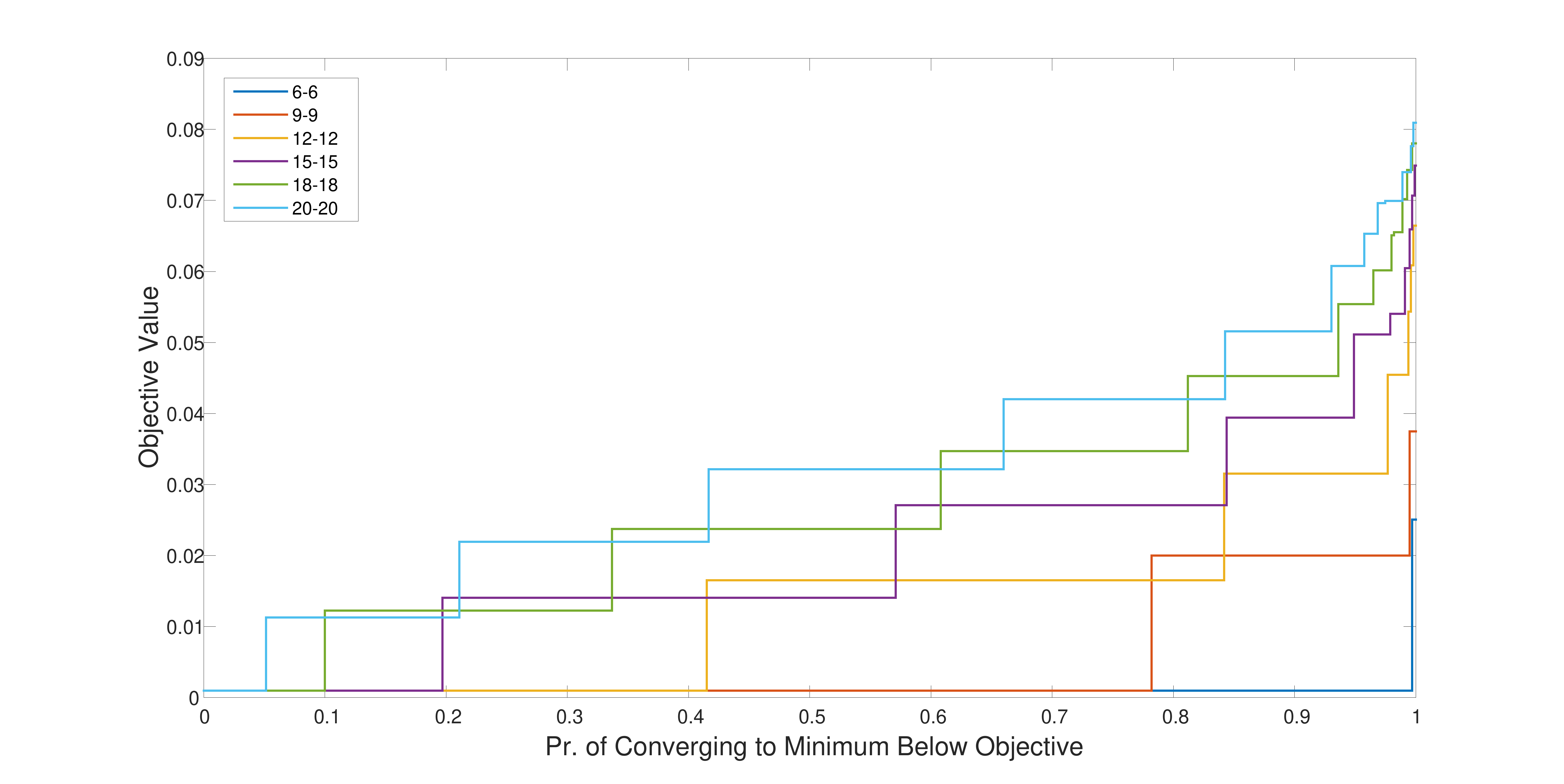}
	\includegraphics[scale=.35]{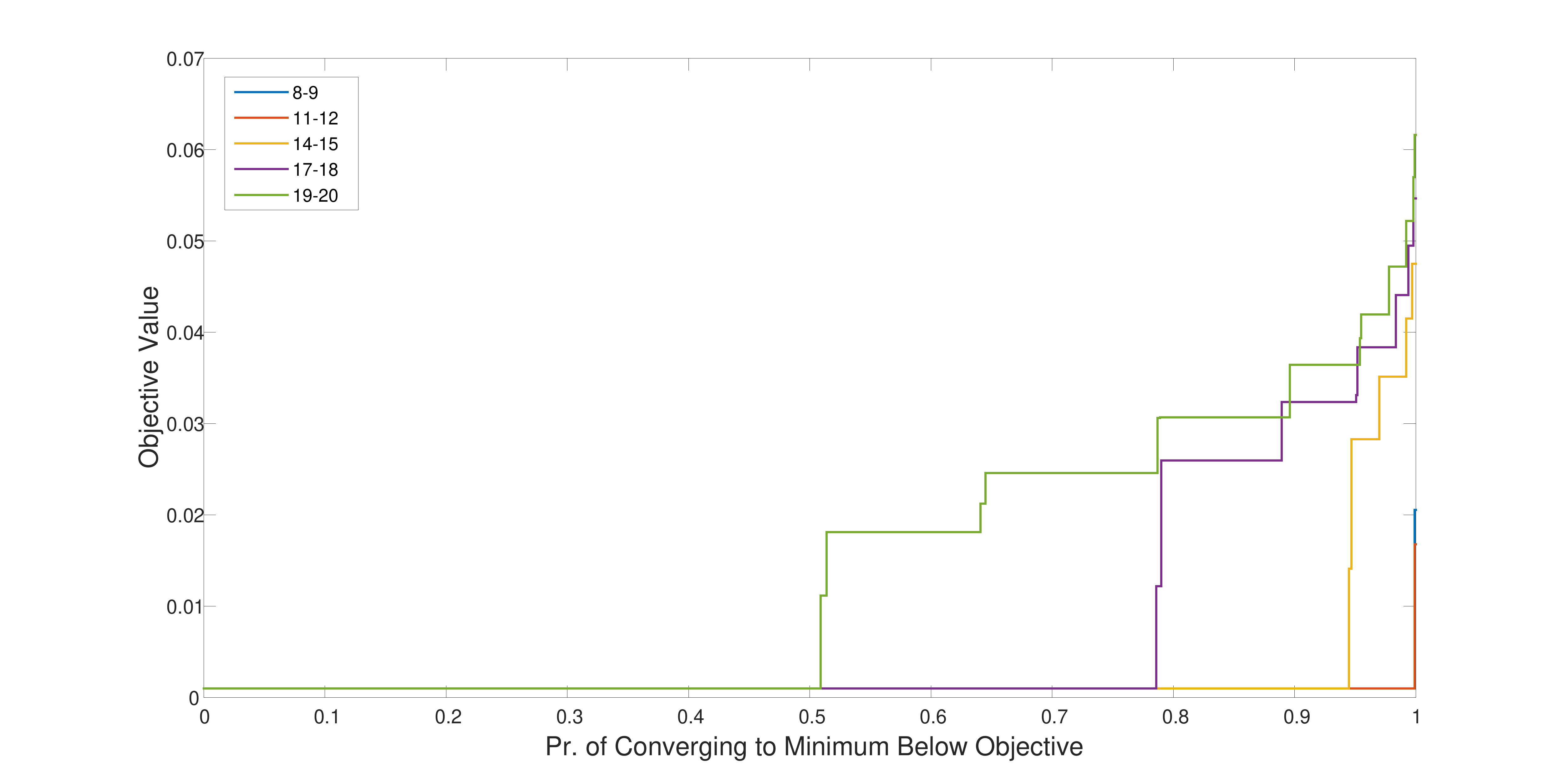}
	
	\caption{The empirical probability of converging to a minimum with 
		objective value smaller than a given quantity, out of the 1000 runs. 
		Different lines correspond to different choices of $(k,n)$.}
	\label{fig:obj_val_cdf}
\par\end{figure}

\begin{example}\label{ex:k8n9}
	Out of 1000 gradient descent initializations for $ k=8,n=9 $, one converged 
	to a local minimum. The point found, 
	denoted $ \bw_1^9 $, is given below:
	\begin{equation*}
		\bw_1^9=\begin{bmatrix}
			
			0.9841 & -0.0298 & -0.0298 & -0.0298 & -0.0298 & -0.0298 & -0.0298 
			& 0.1263 & 0.0687 \\
			-0.0298 & 0.9841 & -0.0298 & -0.0298 & -0.0298 & -0.0298 & -0.0298 
			& 0.1263 & 0.0687 \\
			-0.0298 & -0.0298 & 0.9841 & -0.0298 & -0.0298 & -0.0298 & -0.0298 
			& 0.1263 & 0.0687 \\
			-0.0298 & -0.0298 & -0.0298 & 0.9841 & -0.0298 & -0.0298 & -0.0298 
			& 0.1263 & 0.0687 \\
			-0.0298 & -0.0298 & -0.0298 & -0.0298 & 0.9841 & -0.0298 & -0.0298 
			& 0.1263 & 0.0687 \\
			-0.0298 & -0.0298 & -0.0298 & -0.0298 & -0.0298 & 0.9841 & -0.0298 
			& 0.1263 & 0.0687 \\
			-0.0298 & -0.0298 & -0.0298 & -0.0298 & -0.0298 & -0.0298 & 0.9841 
			& 0.1263 & 0.0687 \\
			0.2301 & 0.2301 & 0.2301 & 0.2301 & 0.2301 & 0.2301 & 0.2301 & 
			-0.1890 & -0.4862
		\end{bmatrix},
	\end{equation*}
	where the parameter vector of each of the $ 9 $ neurons corresponds to a 
	column of $ 
	\bw_1^9 $.
	The Hessian of the objective at $ \bw_1^9 $, $ \nabla^2F\p{\bw_1^9} 
	$, was confirmed to have minimal eigenvalue $ 
	\lambda_{\min}\p{\nabla^2F\p{\bw_1^9}} \ge 0.005944 $. This implied 
	that $ \bw_1^9 $ is of distance at most $ r=7.8\cdot10^{-8} $ from a 
	local minimum with objective value at least $ 0.02056 $.
\end{example}

It is interesting to note that the points found in examples \ref{ex:k6n6} and \ref{ex:k8n9}, as well as all other local minima detected, have a nice symmetric 
structure: We see that most of the trained neurons are very close to the target neurons in most of the dimensions. Also, many of the entries appear to be the same. Surprisingly, although 
such constructions might seem brittle, these are indeed strict local minima. Moreover, the 
probability of converging to such points becomes very large as the network size 
increases as demonstrated by our experiments.

\section{Proofs}\label{sec:proofs}

In the proofs, we use bold-faced letters (e.g., $\bw$) to denote 
vectors, barred bold-faced letters (e.g., $\bar{\bw}$) to denote vectors 
normalized to unit Euclidean norm, and capital letters to generally denote 
matrices. Given a natural 
number $ k $, we let $ [k] $ be shorthand for $\set{1,\dots,k} $. Given a 
matrix $M$, $\spnorm{M}$ denotes its spectral norm. We will also make use of 
the following version of Weyl's inequality, stated below for completeness.

\begin{theorem}[Weyl's inequality]\label{thm:weyl}
	Suppose $ A, B, P\in\reals^{d\times d} $ are real symmetric matrices such 
	that $ A-B=P $. Assume that $ A,B $ have eigenvalues $ 
	\alpha_1\ge\ldots\ge\alpha_d, \beta_1\ge\ldots\ge\beta_d $ respectively, 
	and that $ \spnorm{P} \le \epsilon $. Then
	\[
	\abs{\alpha_i-\beta_i}\le\epsilon ~~~\forall i\in\pcc{d}.
	\]
\end{theorem}

\subsection{Proof of \thmref{thm:main}}

To prove \thmref{thm:main} for some $(k,n)$, it is enough to consider some 
particular choice of orthogonal $\bv_1,\ldots,\bv_k$, since any other choice 
amounts to rotating or reflecting the same objective function (which of course 
does not 
change the existence or non-existence of its local minima). In particular, we 
chose these vectors to simply be the standard basis vectors in $\reals^k$. 

As we show in \subsecref{subsec:closedform} below, the objective function in 
\eqref{eq:objfunover} can be 
written in an explicit form (without the expectation term), as well as its 
gradients and Hessians. We first ran standard gradient descent, starting from 
random initialization and using a fixed step size of $0.1$, till we 
reached a point $\weights$, such that the gradient norm w.r.t. any $\bw_i$ is 
at most $10^{-9}$. Given this point, we use \lemref{lem:key} and 
\lemref{lem:key2} to prove 
that it is close to a local minimum. Specifically, we built code which does the 
following:

\begin{enumerate}
	\item Provide a rigorous upper bound on the norm of the gradient at a given 
	point $ \weights $ (since we have a closed-form expression for the 
	gradient, this only requires elementary calculations).
	\item Provide a rigorous lower bound on the minimal eigenvalue of 
	$\nabla^2F(\weights)$: This is the technically most demanding part, and the derivation 
	of the algorithm is presented in \subsecref{subsec:pdalg}.
	\item Provide a rigorous upper bound $B$ on the remainder term $R_{\weights,\bu}$ 
	(see \subsecref{sec:remainder_bound} for the relevant calculations).
	\item Provide a rigorous Lipschitz bound on the objective $ F\p{\weights} 
	$, establishing \lemref{lem:key2} (see 
	\subsecref{subsec:objective_lipschitzness} for the relevant calculations).
\end{enumerate}

We used MATLAB (version 2017b) to perform all floating-point 
computations, and its associated MATLAB VPA package to perform 
the exact symbolic computations. The code we used is freely available at 
https://github.com/ItaySafran/OneLayerGDconvergence.git. For any candidate local minimum, the 
verification took from less than a minute up to a few hours, depending on the 
size of $k,n$, when running on Intel Xeon E5 processors (ranging from E5-2430 
to E5-2660).
%

\subsubsection{Closed-form Expressions for $ F,\nabla F $ and $ \nabla^2 F $}
\label{subsec:closedform}

For convenience, we will now state closed-form expressions (without an 
expectation)
for the objective function $F$ in \eqref{eq:objfunover}, its gradient and its 
Hessian. These are also the expressions used in the code we built to verify the 
conditions of \lemref{lem:key} and \lemref{lem:key2}. First, we have that
\begin{equation}\label{eq:closedform_obj}
F\p{\weights} ~=~ \frac{1}{2}\sum_{i,j=1}^n f\p{\bw_i,\bw_j} - 
\sum_{\substack{i\in\pcc{n}\\j\in\pcc{k}}} f\p{\bw_i,\bv_j} + 
\frac{1}{2}\sum_{i,j=1}^k f\p{\bv_i,\bv_j},
\end{equation}
where
\begin{align}
	f\p{\bw,\bv} &:= \E_{\bx\sim 
		\Ncal(\mathbf{0},I)}\left[\relu{\bw^{\top}\bx}\relu{\bv^{\top}\bx}\right]\notag\\
	&= \frac{1}{2\pi} 
	\norm{\bw}\norm{\bv}\p{\sin\p{\theta_{\bw,\bv}}+\p{\pi-\theta_{\bw,\bv}}
		\cos\p{\theta_{\bw,\bv}}},
	\label{eq:f_function}
\end{align}
and 
$\theta_{\bw,\bv}:=\cos^{-1}\left(\frac{\bw^\top\bv}{\norm{\bw}\cdot\norm{\bv}}\right)$
	is the angle between two vectors $\bw,\bv$. 
The latter equality in \eqref{eq:f_function} was shown in \citet[section 
2]{cho2009kernel}.

Using the above representation, \citet{brutzkus2017globally} compute the gradient of $ f\p{\bw,\bv} $ with respect to $ \bw $, given by
\begin{equation}\label{eq:g_function}
	g\p{\bw,\bv} \coloneqq \frac{\partial}{\partial\bw}f\p{\bw,\bv} = \frac{1}{2\pi} \p{\norm{\bv} \sin\p{\theta_{\bw,\bv}}\bar{\bw} + \p{\pi-\theta_{\bw,\bv}}\bv}.
\end{equation}
Which implies that $ \nabla F\p{\weights} $, the gradient of the objective with respect to $ \weights $, equals
\begin{equation*}
	\nabla F\p{\weights} = \frac{1}{2}\weights + \sum_{\substack{i,j=1\\i\neq j}}^n \tilde{g}\p{\bw_i,\bw_j} - \sum_{\substack{i\in\pcc{n}\\j\in\pcc{k}}} \tilde{g}\p{\bw_i,\bv_j},
\end{equation*}
where $ \tilde{g}\p{\bw_i,\bu}\in\reals^{kn} $ equals $ g\p{\bw_i,\bu}\in\reals^{k} $ on entries $ k(i-1)+1 $ through $ ki $, and zero elsewhere.
We now provide the Hessian of \eqref{eq:f_function} based on the computation of the gradient in \eqref{eq:g_function} (see \subsecref{sec:hessian_derivation} for the full derivation)
\begin{equation*}
	h_1\p{\bw,\bv} \coloneqq \frac{\partial^2}{\partial\bw^2}f\p{\bw,\bv}
	=\frac{\sin\p{\theta_{\bw,\bv}}\norm{\bv}}{2\pi\norm{\bw}}\p{\bI-\bar{\bw}\bar{\bw}^{\top}+\bar{\bn}_{\bv,\bw}\bar{\bn}_{\bv,\bw}^{\top}},
\end{equation*}
\begin{equation*}
	h_2\p{\bw,\bv}\coloneqq \frac{\partial^2}{\partial\bw\partial\bv}f\p{\bw,\bv} = \frac{1}{2\pi}\p{\p{\pi-\theta_{\bw,\bv}}\bI+\bar{\bn}_{\bw,\bv}\bar{\bv}^{\top}+\bar{\bn}_{\bv,\bw}\bar{\bw}^{\top}},
\end{equation*}
where
\begin{equation}\label{eq:bn_def}
	\bn_{\bv,\bw} = \bar{\bv}-\cos\p{\theta_{\bv,\bw}}\bar{\bw},~~~\bar{\bn}_{\bv,\bw} = \frac{\bn_{\bv,\bw}}{\norm{\bn_{\bv,\bw}}}.
\end{equation}
To formally define the Hessian of $F$ (a $kn\times kn$ matrix), we partition it 
into $n\times n$ blocks, each of size $k\times k$. Define $ 
\tilde{h}_1\p{\bw_i,\bu} \in \reals^{kn\times kn} $ to equal $ h_1\p{\bw_i,\bu} 
$ on the $ i $-th $ d\times d $ diagonal block and zero elsewhere. For $ 
\bw_i,\bw_j $ define $ \tilde{h}_2\p{\bw_i,\bw_j} \in \reals^{kn\times kn} $ to 
equal $ h_2\p{\bw_i,\bw_j} $ on the $ i,j $-th $ k\times k $ block and zero 
elsewhere. We now have that the Hessian is given by
\begin{equation}\label{eq:obj_hessian}
	\nabla^2F\p{\weights} = \frac{1}{2}\bI + \sum_{\substack{i,j=1\\i\neq j}}^{n} \tilde{h}_1\p{\bw_i,\bw_j} - \sum_{\substack{i\in\pcc{n}\\j\in\pcc{k}}} \tilde{h}_1\p{\bw_i,\bv_j} + \sum_{\substack{i,j=1\\i\neq j}}^{n}\tilde{h}_2\p{\bw_i,\bw_j}.
\end{equation}

\subsubsection{Lower bound on $\lambda_{\min}$}\label{subsec:pdalg}

We wish to verify that the Hessian of a point returned by the gradient descent 
algorithm is positive definite, as well as provide a lower bound for its 
smallest eigenvalue, avoiding the possibility of errors due to floating-point 
computations.

Since the Hessians we encounter have relatively small entries and are 
well-conditioned, it turns out that computing the spectral decomposition in 
floating-point arithmetic provides a very good approximation of the true 
spectrum of the matrix. Therefore, instead of performing spectral decomposition 
symbolically from scratch, our algorithmic approach is to use the 
floating-point decomposition, and merely bound its error, using simple 
quantities which are easy to compute symbolically. Specifically, given the 
(floating-point, possibly approximate) decomposition $UDU^\top$ of a matrix 
$A$, we bound the error using the distance of $UDU^\top$ from $A$, as well as 
the distance of $U$ from its projection on the subspace of orthogonal matrices 
given by $ \bar{U}\coloneqq U\p{U^{\top}U}^{-0.5} $. Formally, we use the 
following algorithm (where numerical computations refer to operations 
in floating-point arithmetic):

\begin{algorithm}
\begin{algorithmic}
\STATE \textbf{Input:} Square matrix $ A \in\reals^{d\times d}$.
\STATE \textbf{Output:} A lower bound on the smallest eigenvalue of $ A $ if it 
is 
positive-definite and $ -1 $ otherwise.
\STATE
	- Numerically compute $ A^{\prime} $, a double precision estimate of $ A $.
\STATE
	- Symbolically compute $ \epsilon_1=\norm{A-A^{\prime}}_F $.
\STATE
	- Numerically compute $ U,D\in\reals^{d\times d} $ s.t. $ A^{\prime}\approx 
	UDU^{\top} $, $ D $ is diagonal.
\STATE
	- Symbolically compute $ E=I-U^{\top}U $, $ A^{\prime\prime}=UDU^{\top} $, 
	$ \epsilon_2=\norm{A^{\prime}-A^{\prime\prime}}_F $.
\STATE
	- Symbolically compute an upper bound $ B=1+\norm{U-I}_F $ on $ \spnorm{U} 
	$.
\STATE
	- Symbolically compute an upper bound $ C=\norm{E}_F $ on $ \spnorm{E} $.
\STATE
	- Let $ \lambda_{\min},\lambda_{\max} $ denote the smallest and largest 
	diagonal entries of $ D $ respectively, then symbolically compute $ 
	\epsilon_3 = 
	B^2\p{2\lambda_{\max}\p{\frac{1}{\sqrt{1-C}}-1}+\p{\frac{1}{\sqrt{1-C}}-1}^2}$.
\STATE
	- Return $ \lambda_{\min} - \epsilon_1-\epsilon_2 - \epsilon_3 $ if it is 
	larger than $ 0 $ and $ -1 $ otherwise.
\end{algorithmic}
\end{algorithm} 

\textbf{Algorithm analysis:}
For the purpose of analyzing the algorithm, the following two lemmas will be used.
\begin{lemma}\label{lem:central_binom_id}
	For any natural $ n\ge0 $ we have
	\[
		4^{-n}\sum_{k=0}^{n}\binom{2k}{k}\binom{2n-2k}{n-k}=1.
	\]
\end{lemma}

\begin{proof}
	Clearly, for any $ \abs{x}<1 $ we have
	\begin{equation}\label{eq:geometric_sum}
		\frac{1}{1-x}=\sum_{k=0}^{\infty}x^k.
	\end{equation}
	Using the generalized binomial theorem, we have for any $ \abs{x}<1 $
	\begin{align}
		\frac{1}{\sqrt{1-x}} &= \sum_{k=0}^{\infty}\binom{k-0.5}{k}x^k 
		= \sum_{k=0}^{\infty}\frac{\prod_{i=0}^{k-1}\p{k-i-0.5}}{k!}x^k 
		\nonumber\\
		&= \sum_{k=0}^{\infty}\frac{\prod_{i=0}^{k-1}\p{2k-2i-1}}{2^kk!}x^k 
		\nonumber
		= 
		\sum_{k=0}^{\infty}\frac{2^kk!\prod_{i=0}^{k-1}\p{2k-2i-1}}{4^k\p{k!}^2}x^k
		 \nonumber\\
		&= 
		\sum_{k=0}^{\infty}\frac{\prod_{i=0}^{k-1}\p{2k-2i}\prod_{i=0}^{k-1}\p{2k-2i-1}}{4^k\p{k!}^2}x^k
		= \sum_{k=0}^{\infty}\frac{\p{2k}!}{4^k\p{k!}^2}x^k \nonumber\\
		&= \sum_{k=0}^{\infty}\binom{2k}{k}4^{-k}x^k. \label{eq:sqrt_sum}
	\end{align}
	Consider the $ k $-th coefficient in the expansion of the square of \eqref{eq:sqrt_sum}, which is well defined as the sum converges absolutely for any $ \abs{x}<1 $. From \eqref{eq:geometric_sum}, these coefficients are all $ 1 $. However, these are also given by the expansion of the square of \eqref{eq:sqrt_sum}. Specifically, the $ k $-th coefficient in the square is given as the sum of all $ x^k $ coefficients in the expansion of the root, that is, it is a convolution of the coefficients in \eqref{eq:sqrt_sum} with index $ \le k $, thus we have
	\begin{equation*}
		4^{-n}\sum_{k=0}^{n}\binom{2k}{k}\binom{2n-2k}{n-k}=1.
	\end{equation*}
\end{proof}

\begin{lemma}
	Let $ U^{\top}U $ be a diagonally dominant matrix, let $ E = \bI-U^{\top}U $ satisfying $ \spnorm{E}\le C<1 $. Then $ \p{U^{\top}U}^{-0.5} = \sum_{n=0}^{\infty}\binom{2n}{n}4^{-n}E^n$. Moreover, $ E^{\prime}\coloneqq \sum_{n=1}^{\infty}\binom{2n}{n}4^{-n}E^n $ satisfies $ \spnorm{E^{\prime}}\le\p{\frac{1}{\sqrt{1-C}}-1} $.
\end{lemma}

\begin{proof}
	Consider the series given by the partial sums
	\begin{equation*}
		S_n=\sum_{k=0}^{n}\binom{2k}{k}4^{-k}E^k,
	\end{equation*}
	and observe that
	\begin{align}\label{eq:sqrt_of_inverse}
		U^{\top}US_n^2 &= \p{\bI-E}\p{\sum_{k=0}^{n}\binom{2k}{k}4^{-k}E^k}^2 \nonumber\\
		&= \p{\bI-E}\p{\sum_{k=0}^{n}E^k+\sum_{k=n+1}^{2n}\beta_kE^k} \nonumber\\
		&= \bI-E^{n+1}+\p{\bI-E}E^{n+1}\sum_{k=0}^{n-1}\beta_{n+k+1}E^k,
	\end{align}
	where the second equality is due to \lemref{lem:central_binom_id}, and holds for some $ \beta_k\in\p{0,1} $, $ k\in\set{n+1,\dots,2n} $. Now, since
	\begin{align*}
		&\lim_{n\to\infty}\spnorm{\p{\bI-E}E^{n+1}\sum_{k=0}^{n-1}\beta_{n+k+1}E^k}\\
		\le&\lim_{n\to\infty} \spnorm{\bI-E}\spnorm{E}^{n+1}\spnorm{\sum_{k=0}^{n-1}\beta_{n+k+1}E^k}\\
		\le&\lim_{n\to\infty} \spnorm{\bI-E}\spnorm{E}^{n+1}\p{\sum_{k=0}^{n-1}\beta_{n+k+1}\spnorm{E}^k}\\
		\le&\lim_{n\to\infty} \spnorm{\bI-E}\spnorm{E}^{n+1}\p{\sum_{k=0}^{n-1}C^k}\\
		\le&\lim_{n\to\infty} \spnorm{\bI-E}\spnorm{E}^{n+1}\p{1-C}^{-1}\\
		=&0,
	\end{align*}
	we have that \eqref{eq:sqrt_of_inverse} reduces to $ \bI $ as $ n\to\infty $, concluding the proof of the lemma.
\end{proof}
Turning back to the algorithm analysis, we wish to numerically compute the eigenvalues of $ A $ and bound their deviation due to roundoff errors. Other than the inaccuracy in computing $ A^{\prime\prime}\approx A^{\prime} $, another obstacle is that $ U $ is not exactly orthogonal, however it is very close to orthogonal in the sense that $ E=I-U^{\top}U $ has a small norm.
Let $ \bar{U}=U\p{U^{\top}U}^{-0.5} $ be the projection of $ U $ onto the space of orthogonal matrices in $ \reals^{d\times d} $. Clearly, $ \p{U^{\top}U}^{-0.5} $ is well defined if $ U^{\top}U $ is diagonally-dominant, hence positive-definite, which can be easily verified. Also,
\begin{align*}
\bar{U}^{\top}\bar{U}&=U\p{U^{\top}U}^{-0.5}\p{U\p{U^{\top}U}^{-0.5}}^{\top}\\
&= U\p{U^{\top}U}^{-0.5}\p{U^{\top}U}^{-0.5}U^{\top}\\
&= U\p{U^{\top}U}^{-1}U^{\top}\\
&= UU^{-1}\p{U^{\top}}^{-1}U^{\top}\\
&= \mathbf{I}.
\end{align*}
We now upper bound $ \spnorm{A^{\prime\prime}-\bar{A}} $, where $ \bar{A}=\bar{U}D\bar{U}^{\top} $ and therefore its spectrum is given to us explicitly as the diagonal entries of $ D $, $ \text{diag}\p{D} $.
Compute
\begin{align*}
	\spnorm{A^{\prime\prime}-\bar{A}} &= \spnorm{UDU^{\top}-\bar{U}D\bar{U}^{\top}}\\
	&= \spnorm{UDU^{\top}-U\p{U^{\top}U}^{-0.5}D\p{U\p{U^{\top}U}^{-0.5}}^{\top}}\\
	&= \spnorm{U\p{D-\p{U^{\top}U}^{-0.5}D\p{U^{\top}U}^{-0.5}}U^{\top}}\\
	&= \spnorm{U\p{D-\p{I+E^{\prime}}D\p{I+E^{\prime}}}U^{\top}}\\
	&= \spnorm{U\p{E^{\prime}D+DE^{\prime}+E^{\prime2}}U^{\top}}\\
	&\le \spnorm{U}^2\spnorm{E^{\prime}D+DE^{\prime}+E^{\prime2}}\\
	&\le \spnorm{U}^2\p{2\spnorm{D}\spnorm{E^\prime}+\spnorm{E^{\prime}}^2}\\
	&\le B^2\p{2\lambda_{\max}\p{\frac{1}{\sqrt{1-C}}-1}+\p{\frac{1}{\sqrt{1-C}}-1}^2}\\
	&=\epsilon_3.
\end{align*}
Estimating the spectrum $ \text{diag}\p{D} $ of $ A $ using the spectrum of $ \bar{A} $ yields an approximation error of
\begin{align*}
	\spnorm{A-\bar{A}}
	&= \spnorm{A-A^{\prime}+A^{\prime}-A^{\prime\prime}+A^{\prime\prime}-\bar{A}}\\
	&\le \spnorm{A-A^{\prime}}+\spnorm{A^{\prime}-A^{\prime\prime}}+\spnorm{A^{\prime\prime}-\bar{A}}\\
	&\le \epsilon_1 + \epsilon_2 + \epsilon_3,
\end{align*}
where in the last inequality we used the fact that the Frobenius norm upper bounds the spectral norm, which also proves that $ C $ is an upper bound on $ \spnorm{E^{\prime}} $.
Verifying the upper bound given by $ B $, we compute
\[
	\spnorm{U} = \spnorm{U-I+I} \le \spnorm{U-I}+\spnorm{I} \le 1+\norm{U-I}_F.
\]
Whenever $ U $ is close to unity, this provides a sharper upper bound than taking $ C=\norm{U}_F $.

Finally, applying Weyl's inequality (\thmref{thm:weyl}) to $ A $ and $ \bar{A} 
$, we have that the spectra of the two cannot deviate by more than $ 
\epsilon_1+\epsilon_2+\epsilon_3 $, concluding the proof of the algorithm.

\subsubsection{Upper Bound on Remainder Term $ R_{\weights,\direction} 
$}\label{sec:remainder_bound}

In a nutshell, to derive an upper bound $ L $ on the third order term in 
\eqref{eq:taylor}, we show that the second order term in any direction is $ L 
$-Lipschitz. Recalling that the purpose of this upper bound is to provide the 
radius of the ball enclosing a minimum in the vicinity of $ \weights $ (see 
\lemref{lem:key}), we observe, however, that \lemref{lem:hess_sp_norm_bound} 
suggests $ L $ depends on the norm each neuron attains inside the ball, and 
therefore also on the radius of the ball enclosing the minimum. To circumvent 
this circular dependence between the radius and the third order bound, we first 
fixed the radius around $ \weights $ where we bound the third order 
term\footnote{specifically, the radius was chosen to be a $ 10^{-3} $ fraction 
of $ \max_{i\in\pcc{n}}\norm{\bw_i}_2 $. Testing 
this value, we observed that restricting the radius further only slightly 
improved the bound}, and then checked whether the resulting radius enclosing 
the ball is smaller than the one used for the bound, thus validating the 
result.

In what follows, the ball where the third order bound is derived on is referred to as some compact subset of the weight space $ A\subseteq\reals^{kn} $. We now define some notation that will be used throughout the rest of this section. Given $ A $, define
\[
	\bw_{\min}=\min_{\weights\in A}\min_{i\in\pcc{n}}\norm{\bw_i}_2,
\]
\[
	\bw_{\max}=\max_{\weights\in A}\max_{i\in\pcc{n}}\norm{\bw_i}_2.
\]
That is, $ \bw_{\min} $ and $ \bw_{\max} $ are the neurons with minimal and maximal norm among all possible network weights in the set $ A $, respectively. Similarly, defining $ \bv_{\max} $ to be the target parameter vector with maximal $ 2 $-norm, the necessary bound is now given by the following theorem:


\begin{theorem}\label{thm:third_derivative_ubound}
	Suppose $ \nabla^2 F\p{\cdot} $ is differentiable on $ A\subseteq\reals^{kn} $. Then
	\[
	\sup_{\substack{\weights\in A\\\bu:\norm{\bu}_2=1}}\sum_{i_1,i_2,i_3}\frac{\partial^3}{\partial w_{i_1}\partial w_{i_2}\partial w_{i_3}}F\p{\weights}u_{i_1}u_{i_2}u_{i_3} \le L_A,
	\]
	where
	\[
	L_A\coloneqq\frac{n}{\pi\norm{\bw_{\min}}^2} \p{ \sqrt{2}\p{n-1}\p{\norm{\bw_{\max}}+\norm{\bw_{\min}}} + k\norm{\bv_{\max}}}.
	\]
\end{theorem}
To prove the theorem, we will first need the following two lemmas.

\begin{lemma}\label{lem:h1h2_lipschitz}
	Suppose $ \nabla^2F\p{\cdot} $ is differentiable on $ A\subseteq\reals^{kn} $. Then
	\begin{itemize}
		\item
		$ h_1\p{\bw,\bv} $ is $ \frac{\norm{\bv_{\max}}}{\pi\norm{\bw_{\min}}^2} $ Lipschitz in $ \bw $ on $ A $.
		\item
		$ h_1\p{\bw_1,\bw_2} $ is $ \frac{\sqrt{2}\norm{\bw_{\max}}}{\pi\norm{\bw_{\min}}^2} $ Lipschitz in $ \p{\bw_1,\bw_2} $ on $ A $.
		\item
		$ h_2\p{\bw_1,\bw_2} $ is $ \frac{\sqrt{2}}{\pi\norm{\bw_{\min}}} $ Lipschitz in $ \p{\bw_1,\bw_2} $ on $ A $.
	\end{itemize}
\end{lemma}

\begin{proof}
	We begin with computing some useful derivatives:
	\[
	\frac{\partial}{\partial \bw} \cos\p{\theta_{\bw,\bv}}
	= \frac{\partial}{\partial \bw} \frac{\bw^{\top}\bv}{\norm{\bw}\norm{\bv}} 
	= \frac{\bv}{\norm{\bw}\norm{\bv}} - 
	\frac{\bw}{\norm{\bw}^2}\frac{\bw^{\top}\bv}{\norm{\bw}\norm{\bv}} 
	= \frac{\bn_{\bv,\bw}}{\norm{\bw}}~.
	\]
	
	\begin{align*}
	\frac{\partial}{\partial \bw} \sin\p{\theta_{\bw,\bv}}
	&= \frac{\partial}{\partial \bw} 
	\sqrt{1-\p{\frac{\bw^{\top}\bv}{\norm{\bw}\norm{\bv}}}^2} 
	= 
	\p{-\frac{\frac{\bw^{\top}\bv}{\norm{\bw}\norm{\bv}}}{\sqrt{1-\p{\frac{\bw^{\top}\bv}{\norm{\bw}\norm{\bv}}}^2}}}
	 \frac{\bn_{\bv,\bw}}{\norm{\bw}} \\
	&= 
	-\frac{\cos\p{\theta_{\bw,\bv}}}{\norm{\bw}\sin\p{\theta_{\bw,\bv}}}\bn_{\bv,\bw}
	= -\frac{\cos\p{\theta_{\bw,\bv}}}{\norm{\bw}}\bar{\bn}_{\bv,\bw}.
	\end{align*}
	
	\begin{align*}
	\frac{\partial}{\partial \bw} \theta_{\bw,\bv}
	&= \frac{\partial}{\partial \bw} 
	\arccos\p{\frac{\bw^{\top}\bv}{\norm{\bw}\norm{\bv}}} 
	= -\frac{1}{\sqrt{1-\p{\frac{\bw^{\top}\bv}{\norm{\bw}\norm{\bv}}}^2}} 
	\frac{\bn_{\bv,\bw}}{\norm{\bw}} 
	= -\frac{\bar{\bn}_{\bv,\bw}}{\norm{\bw}}.
	\end{align*}
	Now, differentiating the spectral norms of $ h_1 $ and $ h_2 $ using \lemref{lem:hess_sp_norm_bound} yields
	\begin{align*}
	\frac{\partial}{\partial \bw} \spnorm{h_1\p{\bw,\bv}}
	&= \frac{\partial}{\partial \bw} \frac{\sin\p{\theta_{\bw,\bv}}\norm{\bv}}{\pi\norm{\bw}} \\
	&= -\frac{\cos\p{\theta_{\bw,\bv}}\norm{\bv}}{\pi\norm{\bw}^2}\bar{\bn}_{\bv,\bw} + \frac{\sin\p{\theta_{\bw,\bv}}\norm{\bv}}{\pi\norm{\bw}^2}\bar{\bw}\\
	&= \frac{\norm{\bv}}{\pi\norm{\bw}^2} \p{\sin\p{\theta_{\bw,\bv}}\bar{\bw} - \cos\p{\theta_{\bw,\bv}}\bar{\bn}_{\bv,\bw}}, \\
	\end{align*}
	therefore
	\begin{align*}
	&\norm{\frac{\partial}{\partial \bw} \spnorm{h_1\p{\bw,\bv}}}_2 \\
	=& \norm{\frac{\norm{\bv}}{\pi\norm{\bw}^2} \p{\sin\p{\theta_{\bw,\bv}}\bar{\bw} - \cos\p{\theta_{\bw,\bv}}\bar{\bn}_{\bv,\bw}}}_2 \\
	=& \frac{\norm{\bv}}{\pi\norm{\bw}^2}\sqrt{ \p{\sin\p{\theta_{\bw,\bv}}\bar{\bw} - \cos\p{\theta_{\bw,\bv}}\bar{\bn}_{\bv,\bw}}^{\top} \p{\sin\p{\theta_{\bw,\bv}}\bar{\bw} - \cos\p{\theta_{\bw,\bv}}\bar{\bn}_{\bv,\bw}}} \\
	=& \frac{\norm{\bv}}{\pi\norm{\bw}^2}\sqrt{ \sin^2\p{\theta_{\bw,\bv}}\norm{\bar{\bw}}^2 + \cos^2\p{\theta_{\bw,\bv}}\norm{\bar{\bn}_{\bv,\bw}}^2} \\
	=& \frac{\norm{\bv}}{\pi\norm{\bw}^2}.
	\end{align*}
	Next, differentiating with respect to $ \bv $ gives
	\begin{align*}
	\frac{\partial}{\partial \bv} \spnorm{h_1\p{\bw,\bv}}
	&= \frac{\partial}{\partial \bv} \frac{\sin\p{\theta_{\bw,\bv}}\norm{\bv}}{\pi\norm{\bw}} \\
	&= -\frac{\cos\p{\theta_{\bw,\bv}}\norm{\bv}}{\pi\norm{\bw}\norm{\bv}}\bar{\bn}_{\bw,\bv} + \frac{\sin\p{\theta_{\bw,\bv}}}{\pi\norm{\bw}}\bar{\bv} \\
	&= \frac{1}{\pi\norm{\bw}} \p{\sin\p{\theta_{\bw,\bv}}\bar{\bv} - \cos\p{\theta_{\bw,\bv}}\bar{\bn}_{\bw,\bv}},
	\end{align*}
	so
	\begin{equation*}
	\norm{\frac{\partial}{\partial \bv} \spnorm{h_1\p{\bw,\bv}}}_2
	= \frac{1}{\pi\norm{\bw}} \p{\sin\p{\theta_{\bw,\bv}}\bar{\bv} - \cos\p{\theta_{\bw,\bv}}\bar{\bn}_{\bw,\bv}}
	= \frac{1}{\pi\norm{\bw}}.
	\end{equation*}
	Concluding the derivation for the spectral norm of the gradient of $ h_1 $ we get
	\begin{equation}\label{eq:h1_spnorm_grad_bound}
	\norm{\frac{\partial}{\partial \p{\bw,\bv}} \spnorm{h_1\p{\bw,\bv}}}_2
	= \sqrt{\p{\frac{1}{\pi\norm{\bw}}}^2 + \p{\frac{\norm{\bv}}{\pi\norm{\bw}^2}}^2}
	= \frac{1}{\pi\norm{\bw}^2}\sqrt{\norm{\bw}^2+\norm{\bv}^2}.
	\end{equation}
	Similarly, for $ h_2 $ we have
	\begin{align*}
	\frac{\partial}{\partial \bw} \spnorm{h_2\p{\bw,\bv}}
	&= \frac{\partial}{\partial \bw} \frac{1}{2\pi}\p{\pi-\theta_{\bw,\bv}+\sin\p{\theta_{\bw,\bv}}} \\
	&= \frac{1}{2\pi}\p{\frac{\bar{\bn}_{\bw,\bv}}{\norm{\bw}} - \frac{\cos\p{\theta_{\bw,\bv}}}{\norm{\bw}}\bar{\bn}_{\bv,\bw}} \\
	&= \frac{1-\cos\p{\theta_{\bw,\bv}}}{2\pi\norm{\bw}}\bar{\bn}_{\bv,\bw},
	\end{align*}
	thus
	\begin{align*}
	\norm{\frac{\partial}{\partial \bw} \spnorm{h_2\p{\bw,\bv}}}_2
	&= 
	\norm{\frac{1-\cos\p{\theta_{\bw,\bv}}}{2\pi\norm{\bw}}\bar{\bn}_{\bv,\bw}}_2
	= \frac{1-\cos\p{\theta_{\bw,\bv}}}{2\pi\norm{\bw}}
	\le \frac{1}{\pi\norm{\bw}}.
	\end{align*}
	For the gradient with respect to $ \bv $ we have
	\begin{align*}
	\frac{\partial}{\partial \bv} \spnorm{h_2\p{\bw,\bv}}
	&= \frac{\partial}{\partial \bv} \frac{1}{2\pi}\p{\pi-\theta_{\bw,\bv}+\sin\p{\theta_{\bw,\bv}}} \\
	&= \frac{1}{2\pi}\p{\frac{\bar{\bn}_{\bw,\bv}}{\norm{\bv}} - \frac{\cos\p{\theta_{\bw,\bv}}}{\norm{\bv}}\bar{\bn}_{\bw,\bv}} \\
	&= \frac{1-\cos\p{\theta_{\bw,\bv}}}{2\pi\norm{\bv}}\bar{\bn}_{\bw,\bv},
	\end{align*}
	which implies
	\[
	\norm{\frac{\partial}{\partial \bv} \spnorm{h_2\p{\bw,\bv}}}_2
	= 
	\norm{\frac{1-\cos\p{\theta_{\bw,\bv}}}{2\pi\norm{\bv}}\bar{\bn}_{\bw,\bv}}_2
	 \\
	= \frac{1-\cos\p{\theta_{\bw,\bv}}}{2\pi\norm{\bv}}.
	\]
	Concluding the derivation for the spectral norm of the gradient of $ h_2 $ we get
	\begin{align}\label{eq:h2_spnorm_grad_bound}
	\norm{\frac{\partial}{\partial \p{\bw,\bv}} \spnorm{h_2\p{\bw,\bv}}}_2
	&= \sqrt{\p{\frac{1-\cos\p{\theta_{\bw,\bv}}}{2\pi\norm{\bw}}}^2+\p{\frac{1-\cos\p{\theta_{\bw,\bv}}}{2\pi\norm{\bv}}}^2} \nonumber\\
	&\le \frac{1}{\pi} \sqrt{\frac{1}{\norm{\bw}^2}+\frac{1}{\norm{\bv}^2}}.
	\end{align}
	Finally, since a differentiable function is $ L $-Lipschitz if and only if its gradient's $ 2 $-norm is bounded by $ L $, the lemma follows from substituting $ \bw_{\min},\bw_{\max},\bv_{\max} $ in \eqref{eq:h1_spnorm_grad_bound} and \eqref{eq:h2_spnorm_grad_bound}.
\end{proof}

\begin{lemma}\label{lem:hess_lipschitz}
	Suppose $ \nabla^2 F\p{\cdot} $ is differentiable on $ A\subseteq\reals^{kn} $. Then $ \spnorm{\nabla^2 F\p{\cdot}} $ is $ L_A $-Lipschitz in $ \weights $ on $ A $.
\end{lemma}

\begin{proof}
	Since \lemref{lem:h1h2_lipschitz} implies the Lipschitzness of the spectral norms of $ \tilde{h}_1,\tilde{h}_2 $ in $ \weights\in\reals^{kn} $, we let $ \weights =\p{\bw_1,\dots,\bw_n} $, $ \weightsprime=\p{\bw_1^{\prime},\dots,\bw_n^{\prime}}\in A $, then compute
	\begin{align*}
	& \spnorm{\nabla^2 F\p{\weights} - \nabla^2 F\p{\weightsprime}} \\
	&~~~~~~~= \left|\left|\frac{1}{2}\bI + \sum_{\substack{i,j=1\\i\neq j}}^{n} 
	\tilde{h}_1\p{\bw_i,\bw_j} - \sum_{\substack{i=1,\dots,n\\j=1,\dots,k}} 
	\tilde{h}_1\p{\bw_i,\bv_j} + \sum_{\substack{i,j=1\\i\neq 
	j}}^{n}\tilde{h}_2\p{\bw_i,\bw_j}\right.\right. \\
	&~~~~~~~~~~~~~~~~- \left.\left.\p{\frac{1}{2}\bI + 
	\sum_{\substack{i,j=1\\i\neq 
	j}}^{n} \tilde{h}_1\p{\bw_i^{\prime},\bw_j^{\prime}} - 
	\sum_{\substack{i=1,\dots,n\\j=1,\dots,k}} 
	\tilde{h}_1\p{\bw_i^{\prime},\bv_j} + \sum_{\substack{i,j=1\\i\neq 
	j}}^{n}\tilde{h}_2\p{\bw_i^{\prime},\bw_j^{\prime}}}\right|\right|_{\text{sp}}
	 \\
	&~~~~~~~= \left|\left|\sum_{\substack{i,j=1\\i\neq j}}^{n} 
	\p{\tilde{h}_1\p{\bw_i,\bw_j} - 
	\tilde{h}_1\p{\bw_i^{\prime},\bw_j^{\prime}}} + 
	\sum_{\substack{i=1,\dots,n\\j=1,\dots,k}} \p{\tilde{h}_1\p{\bw_i,\bv_j} - 
	\tilde{h}_1\p{\bw_i^{\prime},\bv_j}} \right.\right. \\
	&~~~~~~~~~~~~~~~~+ \left.\left.\sum_{\substack{i,j=1\\i\neq j}}^{n} 
	\p{\tilde{h}_2\p{\bw_i,\bw_j} - 
	\tilde{h}_2\p{\bw_i^{\prime},\bw_j^{\prime}}}\right|\right|_{\text{sp}} \\
	&~~~~~~~\le \sum_{\substack{i,j=1\\i\neq j}}^{n} 
	\spnorm{\tilde{h}_1\p{\bw_i,\bw_j} - 
	\tilde{h}_1\p{\bw_i^{\prime},\bw_j^{\prime}}} + 
	\sum_{\substack{i=1,\dots,n\\j=1,\dots,k}} 
	\spnorm{\tilde{h}_1\p{\bw_i,\bv_j} - \tilde{h}_1\p{\bw_i^{\prime},\bv_j}} \\
	&~~~~~~~~~~~~~~~~+ \sum_{\substack{i,j=1\\i\neq j}}^{n} 
	\spnorm{\tilde{h}_2\p{\bw_i,\bw_j} - 
	\tilde{h}_2\p{\bw_i^{\prime},\bw_j^{\prime}}} \\
	&~~~~~~~\le \sum_{\substack{i,j=1\\i\neq j}}^{n} 
	\frac{\sqrt{2}\norm{\bw_{\max}}}{\pi\norm{\bw_{\min}}^2}\norm{\weights-\weightsprime}_2
	 + \sum_{\substack{i=1,\dots,n\\j=1,\dots,k}} 
	\frac{\norm{\bv_{\max}}}{\pi\norm{\bw_{\min}}^2}\norm{\weights-\weightsprime}_2 + 
	\sum_{\substack{i,j=1\\i\neq j}}^{n} 
	\frac{\sqrt{2}}{\pi\norm{\bw_{\min}}}\norm{\weights-\weightsprime}_2 \\
	&~~~~~~~= \p{n\p{n-1} 
	\frac{\sqrt{2}\norm{\bw_{\max}}}{\pi\norm{\bw_{\min}}^2} + nk 
	\frac{\norm{\bv_{\max}}}{\pi\norm{\bw_{\min}}^2} + n\p{n-1} 
	\frac{\sqrt{2}}{\pi\norm{\bw_{\min}}}}\norm{\weights-\weightsprime}_2 \\
	&~~~~~~~= \frac{n}{\pi\norm{\bw_{\min}}^2} \p{ 
	\sqrt{2}\p{n-1}\p{\norm{\bw_{\max}}+\norm{\bw_{\min}}} + 
	k\norm{\bv_{\max}}}\norm{\weights-\weightsprime}_2.
	\end{align*}
\end{proof}

\begin{proof}[Proof of \thmref{thm:third_derivative_ubound}]
	Let $ \weights,\weightsprime\in A $. For any $ \bu\in\reals^{kn} $ with $ 
	\norm{\bu}_2=1 $ we have using \lemref{lem:hess_lipschitz}
	\begin{align*}
	&\abs{\bu^{\top}\nabla^2 F\p{\weights}\bu - \bu^{\top}\nabla^2 F\p{\weightsprime}\bu} \\
	=& \abs{\bu^{\top}\p{\nabla^2 F\p{\weights} - \nabla^2 F\p{\weightsprime}}\bu} \\
	\le& \spnorm{\nabla^2 F\p{\weights} - \nabla^2 F\p{\weightsprime}} \\
	\le& L_A\norm{\weights-\weightsprime}_2,
	\end{align*}
	therefore the differentiable on $ A $, $ \reals^{kn}\to \reals $ function $ \weights\mapsto \bu^{\top}\nabla^2 F\p{\weights+t\bu}\bu $ is $ L_A $-Lipschitz for any $ \direction\in\reals^{kn} $, $ \norm{\direction}=1 $, and any $ t $ satisfying $ \weights+t\direction\in A $, hence its derivative on $ A $ is upper bounded by $ L_A $. Namely, we have that
	\[
	\sup_{\substack{\weights\in A\\\bu:\norm{\bu}_2=1}}\sum_{i_1,i_2,i_3}\frac{\partial^3}{\partial w_{i_1}\partial w_{i_2}\partial w_{i_3}}F\p{\weights}u_{i_1}u_{i_2}u_{i_3} \le L_A.
	\]
	
\end{proof}

\subsubsection{Lipschitzness of $ F\p{\weights}$ and Proof of 
\lemref{lem:key2}} 
\label{subsec:objective_lipschitzness}

In this subsection, we turn to proving a Lipschitz bound on the objective in 
\eqref{eq:closedform_obj}, implying \lemref{lem:key2} and showing that the 
local minimum identified in 
\eqref{lem:key} is necessarily non-global. A 
straightforward approach would be to globally upper bound $\norm{\nabla 
F\p{\weights}}$ (excluding the neighborhood of some singular points). However, this approach is quite loose, since it does not take advantage of the fact that 
the gradients $ \nabla F\p{\weights} $ close to our points of interest are very 
small. Instead, we first derive a Lipschitz bound on $ \nabla^2F\p{\weights} 
$, implying that $ \nabla F\p{\weights} $ does not vary too greatly, and 
therefore remains small for any $ \weightsprime $ in the ball enclosing $ 
\weights $, providing a stronger bound than the more naive approach.


\begin{theorem}\label{thm:F_bound}
	Suppose $ F $ is thrice-differentiable on $A\subseteq\reals^{kn} $. Then for any $ \weights,
	\weightsprime\in A $, $$ \abs{F\p{\weightsprime}-F\p{\weights}} \le 
	\norm{\weightsprime-\weights}_2\p{LH 
	\norm{\weightsprime-\weights}_2 + \norm{\nabla F\p{\weights}}_2}, $$ 
	where
	\begin{equation*}
	LH \coloneqq \frac{1}{2} +n\p{n-1} \p{\frac{\norm{\bw_{\max}}}{2\pi\norm{\bw_{\min}}}+\frac{1}{2}} + \frac{nk\norm{\bv_{\max}}}{2\pi\norm{\bw_{\min}}}.
	\end{equation*}
\end{theorem}
To prove the theorem, we will need the following lemma:

\begin{lemma}\label{lem:hess_bounded}
	Suppose $ F $ is thrice-differentiable on $A\subseteq\reals^{kn} $. Then
	$$ \sup_{\weights\in A}\spnorm{\nabla^2F\p{\weights}} \le \frac{1}{2} +n\p{n-1} \p{\frac{\norm{\bw_{\max}}}{2\pi\norm{\bw_{\min}}}+\frac{1}{2}} + \frac{nk\norm{\bv_{\max}}}{2\pi\norm{\bw_{\min}}}. $$
\end{lemma}

\begin{proof}
	Recall the Hessian of the objective as defined in \eqref{eq:obj_hessian}. 
	Using \lemref{lem:hess_sp_norm_bound}, the fact that the spectral norms of 
	$ h_1,h_2 $ and $ \tilde{h}_1,\tilde{h}_2 $ are identical, and the fact 
	that $\sin\p{x}\le x $ for any $ x>0 $, we have for any $ \weights\in A $ 
	\begin{align*}
	\spnorm{\nabla^2F\p{\weights}}
	&= \spnorm{\frac{1}{2}\bI + \sum_{\substack{i,j=1\\i\neq j}}^{n} \tilde{h}_1\p{\bw_i,\bw_j} - \sum_{\substack{i=1,\dots,n\\j=1,\dots,k}} \tilde{h}_1\p{\bw_i,\bv_j} + \sum_{\substack{i,j=1\\i\neq j}}^{n}\tilde{h}_2\p{\bw_i,\bw_j}} \\
	&\le \frac{1}{2} +\sum_{\substack{i,j=1\\i\neq j}}^{n} \spnorm{\tilde{h}_1\p{\bw_i,\bw_j}} + \sum_{\substack{i=1,\dots,n\\j=1,\dots,k}} \spnorm{\tilde{h}_1\p{\bw_i,\bv_j}} + \sum_{\substack{i,j=1\\i\neq j}}^{n}\spnorm{\tilde{h}_2\p{\bw_i,\bw_j}} \\
	&\le \frac{1}{2} +\sum_{\substack{i,j=1\\i\neq j}}^{n} \frac{\sin\p{\theta_{\bw_i,\bw_j}}\norm{\bw_j}}{2\pi\norm{\bw_i}} + \sum_{\substack{i=1,\dots,n\\j=1,\dots,k}} \frac{\sin\p{\theta_{\bw_i,\bv_j}}\norm{\bv_j}}{2\pi\norm{\bw_i}} \\
	&~~~~~~~+ \sum_{\substack{i,j=1\\i\neq j}}^{n}\frac{1}{2\pi}\p{\pi-\theta_{\bw_i,\bw_j}+\sin\p{\theta_{\bw_i,\bw_j}}} \\
	&\le \frac{1}{2} +\sum_{\substack{i,j=1\\i\neq j}}^{n} \frac{\norm{\bw_{\max}}}{2\pi\norm{\bw_{\min}}} + \sum_{\substack{i=1,\dots,n\\j=1,\dots,k}} \frac{\norm{\bv_{\max}}}{2\pi\norm{\bw_{\min}}}
	+ \sum_{\substack{i,j=1\\i\neq j}}^{n}\frac{1}{2\pi}\pi \\
	&\le \frac{1}{2} +n\p{n-1} \p{\frac{\norm{\bw_{\max}}}{2\pi\norm{\bw_{\min}}}+\frac{1}{2}} + \frac{nk\norm{\bv_{\max}}}{2\pi\norm{\bw_{\min}}}.
	\end{align*}
\end{proof}

\begin{proof}[Proof of \thmref{thm:F_bound}]
	For some $ \direction \in\reals^{kn} $, consider the function $ 
	g_{\direction}\p{t}=\direction^{\top}\nabla F\p{\weights + 
	t\p{\weightsprime-\weights}} $
	
	Since $ F $ is thrice-differentiable, we have from the mean value theorem that there exists some $ t_{\direction} $ such that 
	\begin{align*}
	\direction^{\top}\p{\nabla F\p{\weightsprime} - \nabla F\p{\weights}} &= \frac{g_{\direction}\p{1}-g_{\direction}\p{0}}{1-0} \\
	&= g^{\prime}_{\direction}\p{t_{\direction}} \\
	&= \direction^{\top}\nabla^2F\p{\weights + t_{\direction}\p{\weightsprime-\weights}} \p{\weightsprime-\weights}.
	\end{align*}
	Taking $ \direction =\nabla F\p{\weightsprime} + \nabla F\p{\weights} $ and recalling that from \lemref{lem:hess_bounded} we have that $ \sup_{\weightsprime\in A}\spnorm{\nabla^2F\p{\weightsprime}} $ is bounded by $ LH $, we get
	\begin{align*}
	& \norm{\nabla F\p{\weightsprime}}_2^2 - \norm{\nabla F\p{\weights}}_2^2 \\
	=& \p{\nabla F\p{\weightsprime} + \nabla F\p{\weights}}^{\top}\p{\nabla F\p{\weightsprime} - \nabla F\p{\weights}} \\
	=& \p{\nabla F\p{\weightsprime} + \nabla F\p{\weights}}^{\top}\nabla^2F\p{\weights + t_{\direction}\p{\weightsprime-\weights}} \p{\weightsprime-\weights} \\
	\le& \norm{\nabla F\p{\weightsprime} + \nabla F\p{\weights}}_2 \spnorm{\nabla^2F\p{\weights + t_{\direction}\p{\weightsprime-\weights}}}\norm{\weightsprime-\weights}_2 \\
	\le& \p{\norm{\nabla F\p{\weightsprime}}_2 + \norm{\nabla F\p{\weights}}_2} LH \norm{\weightsprime-\weights}_2.
	\end{align*}
	Dividing by $ \norm{\nabla F\p{\weightsprime}}_2 + \norm{\nabla F\p{\weights}}_2 $ and rearranging yields
	\begin{equation*}
	\norm{\nabla F\p{\weightsprime}}_2 \le LH \norm{\weightsprime-\weights}_2 + \norm{\nabla F\p{\weights}}_2.
	\end{equation*}
	That is, the target function $ F $ is $ \p{LH \norm{\weightsprime-\weights}_2 + \norm{\nabla F\p{\weights}}_2} $-Lipschitz on $ A $, thus
	\begin{equation*}
	\abs{F\p{\weightsprime}-F\p{\weights}} \le \norm{\weightsprime-\weights}_2\p{LH \norm{\weightsprime-\weights}_2 + \norm{\nabla F\p{\weights}}_2}.
	\end{equation*}
\end{proof}
\begin{proof}[Proof of \lemref{lem:key2}]
	For $ A $ which is a ball of radius $ r $ centered at $ \weights =\p{\bw_1,\dots,\bw_n} $, we have that $ \norm{\bw_{\max}} = \max_{i}\norm{\bw_i} + r $ as well as $ \norm{\bw_{\min}} = \min_{i}\norm{\bw_i} - r $. Plugging this in \thmref{thm:F_bound} and substituting $ \norm{\weightsprime-\weights}_2 \le r $ completes the proof of the lemma.
\end{proof}

\subsection{Proof of Corollary \ref{cor:main}}\label{subsec:cor_proof}
	
To show the first part of Corollary \ref{cor:main}, we will use the following lemma:
\begin{lemma}\label{lem:same_spectrum}
	Let $ \weights=\p{\bw_1,\dots,\bw_n} $, $ V=\p{\bv_1,\dots,\bv_k} $ where $ \bw_i,\bv_j\in\reals^k $ for all $ i\in\pcc{n},j\in\pcc{k} $. Denote for any natural $ m\ge0 $, $ \weightstilde{m} = \p{\tilde{\bw}_1,\dots,\tilde{\bw}_n} $, $ \tilde{\bw}_i = \p{\bw_i, \mathbf{0}} \in\reals^{k+m} $, $ \tilde{V}_m = \p{\tilde{\bv}_1,\dots,\tilde{\bv}_k} $, $ \tilde{\bv}_i = \p{\bv_i, \mathbf{0}} \in\reals^{k+m} $ and let $ M\in\reals^{n\times n} $ be the matrix with entries  \begin{equation*}
		M_{ij} =  \begin{cases}
		\frac{1}{2} + \sum\limits_{\substack{l=1 \\ l\neq i}}^{n}\frac{\sin\p{\theta_{\bw_i,\bw_l}}\norm{\bw_l}}{2\pi\norm{\bw_i}} - \sum\limits_{l=1}^k \frac{\sin\p{\theta_{\bw_i,\bv_l}}\norm{\bv_l}}{2\pi\norm{\bw_i}}, & i=j \\
		\frac{1}{2\pi}\p{\pi - \theta_{\bw_i,\bw_j}}, & i\neq j
		\end{cases}.
	\end{equation*}
	Then the spectrum of $ \nabla^2 F\p{\weightstilde{m}} $ is comprised of the spectrum of $ \nabla^2 F\p{\weights} $ and the spectrum of $ M $ with multiplicity $ m $. In particular, if $ \nabla^2 F\p{\weightstilde{1}} \succeq \lambda_{\min}\cdot\bI $ then $ \nabla^2 F\p{\weightstilde{m}} \succeq \lambda_{\min}\cdot\bI $, for any $ m>1 $.
\end{lemma}

\begin{proof}
	A straightforward substitution of $ \weightstilde{m} $ and $ \tilde{V}_m $ in \eqref{eq:obj_hessian}, and a permutation of the rows and columns of the resulting matrix reveals that
	\begin{equation*}
		\nabla^2 F\p{\weightstilde{m}}=
		\begin{bmatrix}
			\nabla^2 F\p{\weights} & 0 & 0 & \cdots & 0 \\
			0 & M & 0 & \cdots & 0 \\
			0 & 0 & M & \cdots & 0 \\
			\vdots & \vdots & \vdots & \ddots & \vdots \\
			0 & 0 & 0 & \cdots & M
		\end{bmatrix}.
	\end{equation*}
	Now, diagonalizing the block diagonal $ \nabla^2 F\p{\weightstilde{m}} $ completes the proof of the lemma.
\end{proof}

	Back to the first part of Corollary \ref{cor:main}, we have from 
	\lemref{lem:same_spectrum} that the lower bound on the smallest eigenvalue 
	of $ \nabla^2 F\p{\weightstilde{1}} $ holds for $ \nabla^2 
	F\p{\weightstilde{m}} $ for any $ m\ge1 $. Furthermore, since $ 
	\|\weights\|_2=\|\weightstilde{m}\|_2 $ for any $ m\ge0 $ we have that 
	the upper bound on the third order derivatives from 
	\subsecref{sec:remainder_bound} and the Lipschitz bound on the objective 
	from \subsecref{subsec:objective_lipschitzness} still hold, as well as the 
	bound on the norm of the gradient. Therefore by running the simulations in \secref{sec:experiments} on $ \weightstilde{1} $ instead of $ \weights $, the 
	results apply in any optimization space $ \reals^{n\p{k+m}} $, for 
	natural $ m\ge0 $, since the conditions for invoking \lemref{lem:key} and 
	\lemref{lem:key2} are met with the same exact constants\footnote{Note that for $ m=0 $ the eigenvalue lower bound constant may change, since the spectrum of $ M $ has no impact on the spectrum of $ \nabla^2 F\p{\weightstilde{0}} $. This, however, can only result in a stronger lower bound and does not affect the validity on the results obtained when running the experiments in \secref{sec:experiments} with $ m=1 $.}, completing the first part of the corollary.
	
	For the second part of the corollary, we note that if $\bv_1,\ldots,\bv_k$ 
	are chosen i.i.d. from $\Ncal(\mathbf{0},cI)$, then by standard 
	concentration arguments, for any $\epsilon>0$ and high enough dimension $d$ 
	(depending on $k,\epsilon$), it holds with probability at least 
	$1-\exp(-\Omega(d))$ that 
	$|\frac{1}{\sqrt{cd}}\norm{\bv_i}-1|\leq \epsilon$ and 
	$|\frac{1}{cd}\bv_i^\top\bv_{i'}|\leq \epsilon$ for all 
	$i,i'\in \{1,\ldots,k\}$ (see 
	\citet{ledoux2005concentration}). Therefore, regardless of which 
	distribution we are considering, with probability at least 
	$1-\exp(-\Omega(d))$, we can find a scalar $a>0$ and an orthogonal matrix 
	$M$, such that $\norm{aM\bv_i-\be_i}\leq \epsilon$ for all $i$, where 
	$\be_i$ is the $i$-th standard basis vector. 
	
	Letting $F$ be our objective function (w.r.t. the randomly chosen 
	$\bv_1,\ldots,\bv_k$), and using the rotational 
	symmetry of the Gaussian distribution and the positive-homogeneity of the 
	ReLU function, we have
	\begin{align*}
	F(\weights) &= \frac{1}{2}
	\E_{\bx\sim 
		\Ncal(\mathbf{0},I)}\left[~\left(\sum_{i=1}^{n}[\bw_i^\top\bx]_+
	-\sum_{i=1}^{k}[\bv_i^\top\bx]_+\right)^2~\right]\\
	&=
	\frac{1}{2}\E_{\bx\sim 
	\Ncal(\mathbf{0},I)}\left[~\left(\sum_{i=1}^{n}[\bw_i^\top(M^\top\bx)]_+
	-\sum_{i=1}^{k}[\bv_i^\top(M^\top\bx)]_+\right)^2~\right]\\	
	&=
	\frac{1}{2a^2}\cdot 
	\E_{\bx\sim 
	\Ncal(\mathbf{0},I)}\left[~\left(\sum_{i=1}^{n}[(aM\bw_i)^\top\bx]_+
	-\sum_{i=1}^{k}[(aM\bv_i)^\top\bx]_+\right)^2~\right]~.	
	\end{align*}
	It follows that $F$ has the same local minima as 
	\[
	\tilde{F}(\weights) ~:=~ \frac{1}{2}\E_{\bx\sim 
		\Ncal(\mathbf{0},I)}\left[~\left(\sum_{i=1}^{n}[\bw_i^\top\bx]_+
	-\sum_{i=1}^{k}[(aM\bv_i)^\top\bx]_+\right)^2~\right],
	\]
	since they are equivalent after scaling and rotation
	($\tilde{F}(\weights) = a^2 F(\frac{1}{a}M^\top \weights)$). Thus, it is enough to prove 
	existence of local minima for $\tilde{F}$. 
	
	By the argument above, we can rewrite $\tilde{F}(\weights)$ as
	\[
	\tilde{F}(\weights) := \tilde{F}_{\tilde{\be}_1,\ldots,\tilde{\be}_k}(\weights)~=~ 
	\frac{1}{2}\E_{\bx\sim 
	\Ncal(\mathbf{0},I)}\left[~\left(\sum_{i=1}^{n}[\bw_i^\top\bx]_+
	-\sum_{i=1}^{k}[\tilde{\be}_i^\top\bx]_+\right)^2~\right],
	\]	
	where (with high probability) each $\tilde{\be}_i$ is $\epsilon$-close to 
	the standard basis vector 
	$\be_i$. If $\be_i=\tilde{\be_i}$, we have already shown that there is some 
	local minimum $\weightsstar$, which is in the interior of a sphere $S$ such that 
	$F(\weightsstar)<\min_{\weights\in S}F(\weights)$, and moreover, the ball $B$ enclosed by $S$ does 
	not contain global minima (see \eqref{eq:Scond})
	since \thmref{thm:F_bound} and the condition in \eqref{eq:non_global_cond} imply that the minimal value in the ball enclosing $ \weights $ is strictly positive. In 
	particular, let $\epsilon_0>0$ be such that 
	$F(\weightsstar)<\min_{\weights\in S}F(\weights)-\epsilon_0$ and $\min_{\weights\in 
	B}F(\weights)>\inf_{\weights}F(\weights)+\epsilon_0$. It is easily verified that by 
	setting $\epsilon$ small enough (depending only on $\weightsstar,B,\epsilon_0$ which 
	are all fixed), we can ensure that
	\[
	\max_{\weights\in B}|\tilde{F}_{\tilde{\be}_1,\ldots,\tilde{\be}_k}(\weights)
	-\tilde{F}_{\be_1,\ldots,\be_k}(\weights)| \leq \frac{\epsilon_0}{3},
	\]
	and therefore
	\[
	\tilde{F}_{\tilde{\be}_1,\ldots,\tilde{\be}_k}(\weightsstar)<
	\min_{\weights\in S} \tilde{F}_{\tilde{\be}_1,\ldots,\tilde{\be}_k}(\weights),
	\]
	as well as
	\[
	\min_{\weights\in 
	B}\tilde{F}_{\tilde{\be}_1,\ldots,\tilde{\be}_k}(\weights)>
\inf_{\weights}\tilde{F}_{\tilde{\be}_1,\ldots,\tilde{\be}_k}(\weights),
	\]
	which implies that any minimizer of 
	$\tilde{F}_{\tilde{\be}_1,\ldots,\tilde{\be}_k}$ over $B$ must be a local 
	(non-global) minimum.
	
\subsection{Technical Proofs}\label{sec:technical_proofs}

\subsubsection{Derivation of $ \nabla^2F\p{\weights} 
$}\label{sec:hessian_derivation}

\begin{theorem}
	The Hessian of $ F $ at point $ \weights = \p{\bw_1,\dots,\bw_n} $ with respect to target values $ \p{\bv_1,\dots,\bv_k} $ is given on the main diagonals 
	\begin{equation*}
		\frac{\partial^2F}{\partial\bw_i^2} = \frac{1}{2}\mathbf{I} + \sum_{\substack{j=1\\j\neq i}}^{n} h_1\p{\bw_i,\bw_j} - \sum_{j=1}^k h_1\p{\bw_i,\bv_j},
	\end{equation*}
	and on the off-diagonals by
	\[
	\frac{\partial^2F}{\partial\bw_i\partial\bw_j}=h_2\p{\bw_i,\bw_j},
	\]
	where
	\[
	h_1\p{\bw,\bv}
	=\frac{\sin\p{\theta_{\bw,\bv}}\norm{\bv}}{2\pi\norm{\bw}}\p{\mathbf{I}-\bar{\bw	}\bar{\bw}^{\top}+\bar{\bn}_{\bv,\bw}\bar{\bn}_{\bv,\bw}^{\top}},
	\]
	and
	\[
	h_2\p{\bw,\bv}= \frac{1}{2\pi}\p{\p{\pi-\theta_{\bw,\bv}}\bI+\bar{\bn}_{\bw,\bv}\bar{\bv}^{\top}+\bar{\bn}_{\bv,\bw}\bar{\bw}^{\top}}.
	\]
\end{theorem}

\begin{proof}
	By a straightforward calculation, we have
	\begin{align*}
	\frac{\partial^2 f\p{\bw,\bv}}{\partial w_i^2}
	&=\frac{1}{2\pi}\left(\norm{\bv}\p{\p{\frac{1}{\norm{\bw}}-\frac{w_i^2}{\norm{\bw}^3}}\sin\p{\theta_{\bw,\bv}}-\frac{w_i}{\norm{\bw}}\frac{\frac{\bw^{\top}\bv}{\norm{\bw}\norm{\bv}}}{\sqrt{1-\p{\frac{\bw^{\top}\bv}{\norm{\bw}\norm{\bv}}}^2}}\p{\frac{v_i}{\norm{\bw}\norm{\bv}}-\frac{w_i}{\norm{\bw}^2}\frac{\bw^{\top}\bv}{\norm{\bw}\norm{\bv}}}}\right. \\
	&~~~~~~~~+ \left.\frac{v_i}{\sqrt{1-\p{\frac{\bw^{\top}\bv}{\norm{\bw}\norm{\bv}}}^2}}\p{\frac{v_i}{\norm{\bw}\norm{\bv}}-\frac{w_i}{\norm{\bw}^2}\frac{\bw^{\top}\bv}{\norm{\bw}\norm{\bv}}}\right)\\
	&= \frac{1}{2\pi}\p{\frac{\norm{\bv}}{\norm{\bw}}\sin\p{\theta_{\bw,\bv}}+w_i^2\frac{\norm{\bv}\cos\p{2\theta_{\bw,\bv}}}{\norm{\bw}^3\sin\p{\theta_{\bw,\bv}}}-2w_iv_i\frac{\cos\p{\theta_{\bw,\bv}}}{\norm{\bw}^2\sin\p{\theta_{\bw,\bv}}}+v_i^2\frac{1}{\norm{\bw}\norm{\bv}\sin\p{\theta_{\bw,\bv}}}}
	\end{align*}
	
	\begin{align*}
	\frac{\partial^2 f\p{\bw,\bv}}{\partial w_iw_j}
	&=\frac{1}{2\pi}\left(\norm{\bv}\p{-\frac{w_iw_j}{\norm{\bw}^3}\sin\p{\theta_{\bw,\bv}}-\frac{w_i}{\norm{\bw}}\frac{\frac{\bw^{\top}\bv}{\norm{\bw}\norm{\bv}}}{\sqrt{1-\p{\frac{\bw^{\top}\bv}{\norm{\bw}\norm{\bv}}}^2}}\p{\frac{v_j}{\norm{\bw}\norm{\bv}}-\frac{w_j}{\norm{\bw}^2}\frac{\bw^{\top}\bv}{\norm{\bw}\norm{\bv}}}}\right. \\
	&~~~~~~~~+ \left.\frac{v_i}{\sqrt{1-\p{\frac{\bw^{\top}\bv}{\norm{\bw}\norm{\bv}}}^2}}\p{\frac{v_j}{\norm{\bw}\norm{\bv}}-\frac{w_j}{\norm{\bw}^2}\frac{\bw^{\top}\bv}{\norm{\bw}\norm{\bv}}}\right)\\
	&= \frac{1}{2\pi}\p{w_iw_j\frac{\norm{\bv}\cos\p{2\theta_{\bw,\bv}}}{\norm{\bw}^3\sin\p{\theta_{\bw,\bv}}}-\p{w_iv_j+w_jv_i}\frac{\cos\p{\theta_{\bw,\bv}}}{\norm{\bw}^2\sin\p{\theta_{\bw,\bv}}}+v_iv_j\frac{1}{\norm{\bw}\norm{\bv}\sin\p{\theta_{\bw,\bv}}}}
	\end{align*}
	
	Hence
	
	\begin{align}\label{eq:main_diag_hess}
	\frac{\partial^2 f\p{\bw,\bv}}{\partial \bw^2}
	&= \frac{1}{2\pi}\left(\frac{\norm{\bv}}{\norm{\bw}}\sin\p{\theta_{\bw,\bv}}\bI+\frac{\norm{\bv}\cos\p{2\theta_{\bw,\bv}}}{\norm{\bw}^3\sin\p{\theta_{\bw,\bv}}}\bw\bw^{\top}\right.\\
	& ~~~~~~~~ \left.- \frac{\cos\p{\theta_{\bw,\bv}}}{\norm{\bw}^2\sin\p{\theta_{\bw,\bv}}}\p{\bw\bv^{\top}+\bv\bw^{\top}}+\frac{1}{\norm{\bw}\norm{\bv}\sin\p{\theta_{\bw,\bv}}}\bv\bv^{\top}\right)\nonumber\\
	&= \frac{\norm{\bv}}{2\pi\norm{\bw}}\p{\sin\p{\theta_{\bw,\bv}}\mathbf{I}+\frac{\cos\p{2\theta_{\bw,\bv}}}{\sin\p{\theta_{\bw,\bv}}}\bar{\bw}\bar{\bw}^{\top}-\frac{\cos\p{\theta_{\bw,\bv}}}{\sin\p{\theta_{\bw,\bv}}}\p{\bar{\bw}\bar{\bv}^{\top}+\bar{\bv}\bar{\bw}^{\top}}+\frac{1}{\sin\p{\theta_{\bw,\bv}}}\bar{\bv}\bar{\bv}^{\top}}\nonumber\\
	&= \frac{\norm{\bv}}{2\pi\sin\p{\theta_{\bw,\bv}}\norm{\bw}}\p{\sin^2\p{\theta_{\bw,\bv}}\p{\mathbf{I}-\bar{\bw}\bar{\bw}^{\top}}+\p{\bar{\bv}-\cos\p{\theta_{\bw,\bv}}\bar{\bw}}\p{\bar{\bv}-\cos\p{\theta_{\bw,\bv}}\bar{\bw}}^{\top}}.
	\end{align}
	Recall the definition of $ \bn $ in \eqref{eq:bn_def}, we have that
	\begin{align*}
	\norm{\bn}^2 &= \p{\bar{\bv}-\cos\p{\theta_{\bw,\bv}}\bar{\bw}}^{\top}\p{\bar{\bv}-\cos\p{\theta_{\bw,\bv}}\bar{\bw}} \\
	&= \bar{\bv}^{\top}\bar{\bv}-2\cos\p{\theta_{\bw,\bv}}\bar{\bv}^{\top}\bar{\bw} + \cos^2\p{\theta_{\bw,\bv}}\bar{\bw}^{\top}\bar{\bw} \\
	&= 1-\cos^2\p{\theta_{\bw,\bv}} \\
	&= \sin^2\p{\theta_{\bw,\bv}}.
	\end{align*}
	Therefore \eqref{eq:main_diag_hess} can be written as
	\[
	\frac{\partial^2 f\p{\bw,\bv}}{\partial \bw^2}
	=\frac{\sin\p{\theta_{\bw,\bv}}\norm{\bv}}{2\pi\norm{\bw}}\p{\mathbf{I}-\bar{\bw}\bar{\bw}^{\top}+\bar{\bn}_{\bv,\bw}\bar{\bn}_{\bv,\bw}^{\top}}.
	\]
	Differentiating with respect to different individual parameter vectors, we 
	have
	\begin{align*}
	\frac{\partial^2 f\p{\bw,\bv}}{\partial w_i\partial v_i}
	&=\frac{1}{2\pi}\left(\frac{w_i}{\norm{\bw}}\p{\frac{v_i}{\norm{\bv}}\sin\p{\theta_{\bw,\bv}}-\norm{\bv}\frac{\frac{\bw^{\top}\bv}{\norm{\bw}\norm{\bv}}}{\sqrt{1-\p{\frac{\bw^{\top}\bv}{\norm{\bw}\norm{\bv}}}^2}}\p{\frac{w_i}{\norm{\bw}\norm{\bv}}-\frac{v_i}{\norm{\bv}^2}\frac{\bw^{\top}\bv}{\norm{\bw}\norm{\bv}}}}\right.\\
	&~~~~~~~~+ \left.\p{\pi-\theta_{\bw,\bv}}+\frac{v_i}{\sqrt{1-\p{\frac{\bw^{\top}\bv}{\norm{\bw}\norm{\bv}}}^2}}\p{\frac{w_i}{\norm{\bw}\norm{\bv}}-\frac{v_i}{\norm{\bv}^2}\frac{\bw^{\top}\bv}{\norm{\bw}\norm{\bv}}}\right)\\
	&= \frac{1}{2\pi}\left(-w_i^2\frac{\cos\p{\theta_{\bw,\bv}}}{\norm{\bw^2}\sin\p{\theta_{\bw,\bv}}}+w_iv_i\frac{1}{\norm{\bw}\norm{\bv}}\p{\sin\p{\theta_{\bw,\bv}}+\frac{1}{\sin\p{\theta_{\bw,\bv}}}+\frac{\cos^2\p{\theta_{\bw,\bv}}}{\sin\p{\theta_{\bw,\bv}}}}\right.\\
	&\left.-v_i^2\frac{\cos\p{\theta_{\bw,\bv}}}{\norm{\bv}^2\sin\p{\theta_{\bw,\bv}}}+\p{\pi-\theta_{\bw,\bv}}\right)\\
	&= \frac{1}{2\pi}\left(-w_i^2\frac{\cos\p{\theta_{\bw,\bv}}}{\norm{\bw^2}\sin\p{\theta_{\bw,\bv}}}+w_iv_i\frac{2}{\norm{\bw}\norm{\bv}\sin\p{\theta_{\bw,\bv}}}-v_i^2\frac{\cos\p{\theta_{\bw,\bv}}}{\norm{\bv}^2\sin\p{\theta_{\bw,\bv}}}+\p{\pi-\theta_{\bw,\bv}}\right)
	\end{align*}
	
	\begin{align*}
	\frac{\partial^2 f\p{\bw,\bv}}{\partial w_i\partial v_j}
	&=\frac{1}{2\pi}\left(\frac{w_i}{\norm{\bw}}\p{\frac{v_j}{\norm{\bv}}\sin\p{\theta_{\bw,\bv}}-\norm{\bv}\frac{\frac{\bw^{\top}\bv}{\norm{\bw}\norm{\bv}}}{\sqrt{1-\p{\frac{\bw^{\top}\bv}{\norm{\bw}\norm{\bv}}}^2}}\p{\frac{w_j}{\norm{\bw}\norm{\bv}}-\frac{v_j}{\norm{\bv}^2}\frac{\bw^{\top}\bv}{\norm{\bw}\norm{\bv}}}}\right.\\
	&~~~~~~~~+ \left.\frac{v_i}{\sqrt{1-\p{\frac{\bw^{\top}\bv}{\norm{\bw}\norm{\bv}}}^2}}\p{\frac{w_j}{\norm{\bw}\norm{\bv}}-\frac{v_j}{\norm{\bv}^2}\frac{\bw^{\top}\bv}{\norm{\bw}\norm{\bv}}}\right)\\
	&= \frac{1}{2\pi}\left(-w_iw_j\frac{\cos\p{\theta_{\bw,\bv}}}{\norm{\bw^2}\sin\p{\theta_{\bw,\bv}}}+w_iv_j\frac{1}{\norm{\bw}\norm{\bv}}\p{\sin\p{\theta_{\bw,\bv}}+\frac{\cos^2\p{\theta_{\bw,\bv}}}{\sin\p{\theta_{\bw,\bv}}}}\right.\\
	&\left.~~~~~~~~+w_jv_i\frac{1}{\norm{\bw}\norm{\bv}\sin\p{\theta_{\bw,\bv}}}-v_iv_j\frac{\cos\p{\theta_{\bw,\bv}}}{\norm{\bv}^2\sin\p{\theta_{\bw,\bv}}}\right)\\
	&= -\frac{1}{2\pi}\left(w_iw_j\frac{\cos\p{\theta_{\bw,\bv}}}{\norm{\bw^2}\sin\p{\theta_{\bw,\bv}}}-w_iv_j\frac{1}{\norm{\bw}\norm{\bv}\sin\p{\theta_{\bw,\bv}}}\right.\\
	&\left.~~~~~~~~-w_jv_i\frac{1}{\norm{\bw}\norm{\bv}\sin\p{\theta_{\bw,\bv}}}+v_iv_j\frac{\cos\p{\theta_{\bw,\bv}}}{\norm{\bv}^2\sin\p{\theta_{\bw,\bv}}}\right).
	\end{align*}
	Hence
	\begin{align*}
	\frac{\partial^2 f\p{\bw,\bv}}{\partial \bw\partial \bv}
	&= \p{\frac{\pi-\theta_{\bw,\bv}}{2\pi}}\mathbf{I}+\frac{1}{2\pi\sin\p{\theta_{\bw,\bv}}}\p{\p{\bar{\bw}+\bar{\bv}}\p{\bar{\bw}+\bar{\bv}}^{\top}-\p{1+\cos\p{\theta_{\bw,\bv}}}\p{\bar{\bw}\bar{\bw}^{\top}+\bar{\bv}\bar{\bv}^{\top}}}\\
	&= \p{\frac{\pi-\theta_{\bw,\bv}}{2\pi}}\mathbf{I}+\frac{1}{2\pi\sin\p{\theta_{\bw,\bv}}}\p{\bar{\bw}\bar{\bv}^{\top}+\bar{\bv}\bar{\bw}^{\top}-\cos\p{\theta_{\bw,\bv}}\bar{\bw}\bar{\bw}^{\top}-\cos\p{\theta_{\bw,\bv}}\bar{\bv}\bar{\bv}^{\top}}\\
	&= \p{\frac{\pi-\theta_{\bw,\bv}}{2\pi}}\mathbf{I}+\frac{1}{2\pi\sin\p{\theta_{\bw,\bv}}}\p{\p{\bar{\bw}-\cos\p{\theta_{\bw,\bv}}\bar{\bv}}\bar{\bv}^{\top}+\p{\bar{\bv}-\cos\p{\theta_{\bw,\bv}}\bar{\bw}}\bar{\bw}^{\top}}\\
	&=
	\frac{1}{2\pi}\p{\p{\pi-\theta_{\bw,\bv}}\bI+\bar{\bn}_{\bw,\bv}\bar{\bv}^{\top}+\bar{\bn}_{\bv,\bw}\bar{\bw}^{\top}}.
	\end{align*}
	Recall the objective in \eqref{eq:closedform_obj}, we have that its Hessian is comprised of $ n\times n $ blocks of size $ d\times d $ each. On the main diagonal we therefore have
	\begin{align*}
	\frac{\partial^2F}{\partial\bw_i^2}&=\frac{\partial^2}{\partial\bw_i^2}\p{\frac{1}{2}f\p{\bw_i,\bw_i} + \sum_{\substack{j=1\\j\neq i}}^{n}f\p{\bw_i,\bw_j} + \sum_{j=1}^kf\p{\bw_i,\bv_j}}\\
	&=\frac{1}{2}\bI + \sum_{\substack{j=1\\j\neq i}}^{n} h_1\p{\bw_i,\bw_j} - \sum_{j=1}^k h_1\p{\bw_i,\bv_j},
	\end{align*}
	and on the off diagonal we have
	\[
	\frac{\partial^2F}{\partial\bw_i\partial\bw_j}=h_2\p{\bw_i,\bw_j}.
	\]
\end{proof}
	
\subsubsection{The Spectral Norm of $ h_1 $ and $ h_2 $}

\begin{lemma}\label{lem:hess_sp_norm_bound}
	We have that
	\begin{itemize}
		\item
		$ \spnorm{h_1\p{\bw,\bv}} = 
		\frac{\sin\p{\theta_{\bw,\bv}}\norm{\bv}}{\pi\norm{\bw}}. $
		\item
		$ \spnorm{h_2\p{\bw,\bv}} = \frac{1}{2\pi}\p{\pi - \theta_{\bw,\bv} + 
		\sin\p{\theta_{\bw,\bv}}}. $
	\end{itemize}
\end{lemma}

\begin{proof}
	To find the spectral norm, we compute the spectra of $ h_1, h_2 $.
	\begin{itemize}
		\item
		Clearly, for any $ \bu\in\reals^d $ orthogonal to both $ \bar{\bw},\bar{\bn}_{\bv,\bw} $ we have
		\begin{equation*}
		h_1\p{\bw,\bv}\bu = \frac{\sin\p{\theta_{\bw,\bv}}\norm{\bv}}{2\pi\norm{\bw}}\p{\bI-\bar{\bw}\bar{\bw}^{\top}+\bar{\bn}_{\bv,\bw}\bar{\bn}_{\bv,\bw}^{\top}}\bu = \frac{\sin\p{\theta_{\bw,\bv}}\norm{\bv}}{2\pi\norm{\bw}}\bu.
		\end{equation*} 
		Thus $ \frac{\sin\p{\theta_{\bw,\bv}}\norm{\bv}}{2\pi\norm{\bw}} $ is an eigenvalue of $ h_1 $ with multiplicity at least $ d-2 $. Since $ \bar{\bw},\bar{\bn}_{\bv,\bw} $ are orthogonal, their corresponding eigenvalues comprise the rest of the spectrum of $ h_1 $. Compute
		\begin{align*}
		h_1\p{\bw,\bv}\bar{\bw} &= \frac{\sin\p{\theta_{\bw,\bv}}\norm{\bv}}{2\pi\norm{\bw}}\p{\bI-\bar{\bw}\bar{\bw}^{\top}+\bar{\bn}_{\bv,\bw}\bar{\bn}_{\bv,\bw}^{\top}}\bar{\bw}\\
		&= \frac{\sin\p{\theta_{\bw,\bv}}\norm{\bv}}{2\pi\norm{\bw}}\p{\bar{\bw}-\bar{\bw}\norm{\bar{\bw}}^2}\\
		&= 0.
		\end{align*}
		Hence $ 0 $ is the eigenvalue of $ \bar{\bw} $.
		Also,
		\begin{align*}
		h_1\p{\bw,\bv}\bar{\bn}_{\bv,\bw} &= \frac{\sin\p{\theta_{\bw,\bv}}\norm{\bv}}{2\pi\norm{\bw}}\p{\bI-\bar{\bw}\bar{\bw}^{\top}+\bar{\bn}_{\bv,\bw}\bar{\bn}_{\bv,\bw}^{\top}}\bar{\bn}_{\bv,\bw}\\
		&= \frac{\sin\p{\theta_{\bw,\bv}}\norm{\bv}}{2\pi\norm{\bw}}\p{\bar{\bn}_{\bv,\bw}+\bar{\bn}_{\bv,\bw}\norm{\bar{\bn}_{\bv,\bw}}^2}\\
		&= \frac{\sin\p{\theta_{\bw,\bv}}\norm{\bv}}{\pi\norm{\bw}}\bar{\bn}_{\bv,\bw}.
		\end{align*}
		Therefore $ \frac{\sin\p{\theta_{\bw,\bv}}\norm{\bv}}{\pi\norm{\bw}} $ is the largest eigenvalue of $ h_1 $.
		\item
		Once again, for any $ \bu\in\reals^d $ orthogonal to both $ \bar{\bv},\bar{\bw} $ we have
		\begin{equation*}
		h_2\p{\bw,\bv}\bu = \frac{1}{2\pi}\p{\p{\pi-\theta_{\bw,\bv}}\bI+\bar{\bn}_{\bw,\bv}\bar{\bv}^{\top}+\bar{\bn}_{\bv,\bw}\bar{\bw}^{\top}}\bu = \frac{1}{2\pi}\p{\pi-\theta_{\bw,\bv}}\bu.
		\end{equation*}
		Thus $ \frac{1}{2\pi}\p{\pi-\theta_{\bw,\bv}} $ is an eigenvalue of $ h_2 $ with multiplicity at least $ d-2 $. We now show the remaining two eigenvalues correspond to the eigenvectors $ \bar{\bn}_{\bw,\bv} + \bar{\bn}_{\bv,\bw} $ and $ \bar{\bn}_{\bw,\bv} - \bar{\bn}_{\bv,\bw} $.
		\begin{align*}
		& h_2\p{\bw,\bv}\p{\bar{\bn}_{\bw,\bv} - \bar{\bn}_{\bv,\bw}} \\
		=& \frac{1}{2\pi}\p{\p{\pi-\theta_{\bw,\bv}}\bI+\bar{\bn}_{\bw,\bv}\bar{\bv}^{\top}+\bar{\bn}_{\bv,\bw}\bar{\bw}^{\top}}\p{\bar{\bn}_{\bw,\bv} - \bar{\bn}_{\bv,\bw}}\\
		=& \frac{1}{2\pi}\p{\p{\pi-\theta_{\bw,\bv}}\p{\bar{\bn}_{\bw,\bv} - \bar{\bn}_{\bv,\bw}}+\bar{\bn}_{\bv,\bw}\bar{\bw}^{\top}\bar{\bn}_{\bw,\bv}-\bar{\bn}_{\bw,\bv}\bar{\bv}^{\top}\bar{\bn}_{\bv,\bw}}\\
		=& \frac{1}{2\pi}\p{\p{\pi-\theta_{\bw,\bv}}\p{\bar{\bn}_{\bw,\bv} - \bar{\bn}_{\bv,\bw}}+\bar{\bn}_{\bv,\bw}\bar{\bw}^{\top}\frac{\bar{\bw}-\cos\p{\theta_{\bw,\bv}}\bar{\bv}}{\sin\p{\theta_{\bw,\bv}}}-\bar{\bn}_{\bw,\bv}\bar{\bv}^{\top}\frac{\bar{\bv}-\cos\p{\theta_{\bw,\bv}}\bar{\bw}}{\sin\p{\theta_{\bw,\bv}}}}\\
		=& \frac{1}{2\pi}\p{\p{\pi-\theta_{\bw,\bv}}\p{\bar{\bn}_{\bw,\bv} - \bar{\bn}_{\bv,\bw}}+\bar{\bn}_{\bv,\bw}\frac{1-\cos^2\p{\theta_{\bw,\bv}}}{\sin\p{\theta_{\bw,\bv}}}-\bar{\bn}_{\bw,\bv}\frac{1-\cos^2\p{\theta_{\bw,\bv}}}{\sin\p{\theta_{\bw,\bv}}}}\\
		=& \frac{1}{2\pi}\p{\p{\pi-\theta_{\bw,\bv}}\p{\bar{\bn}_{\bw,\bv} - \bar{\bn}_{\bv,\bw}}-\sin\p{\theta_{\bw,\bv}}\p{\bar{\bn}_{\bw,\bv} - \bar{\bn}_{\bv,\bw}}}\\
		=& \frac{1}{2\pi}\p{\pi-\theta_{\bw,\bv}-\sin\p{\theta_{\bw,\bv}}}\p{\bar{\bn}_{\bw,\bv} - \bar{\bn}_{\bv,\bw}}.
		\end{align*}
		Hence $ \frac{1}{2\pi}\p{\pi-\theta_{\bw,\bv}-\sin\p{\theta_{\bw,\bv}}} $ is an eigenvalue of $ h_2 $. Similarly, we have
		\begin{align*}
		& h_2\p{\bw,\bv}\p{\bar{\bn}_{\bw,\bv} + \bar{\bn}_{\bv,\bw}} \\
		=& \frac{1}{2\pi}\p{\p{\pi-\theta_{\bw,\bv}}\bI+\bar{\bn}_{\bw,\bv}\bar{\bv}^{\top}+\bar{\bn}_{\bv,\bw}\bar{\bw}^{\top}}\p{\bar{\bn}_{\bw,\bv} + \bar{\bn}_{\bv,\bw}}\\
		=& \frac{1}{2\pi}\p{\p{\pi-\theta_{\bw,\bv}}\p{\bar{\bn}_{\bw,\bv} + \bar{\bn}_{\bv,\bw}}+\bar{\bn}_{\bv,\bw}\bar{\bw}^{\top}\bar{\bn}_{\bw,\bv}+\bar{\bn}_{\bw,\bv}\bar{\bv}^{\top}\bar{\bn}_{\bv,\bw}}\\
		=& \frac{1}{2\pi}\p{\p{\pi-\theta_{\bw,\bv}}\p{\bar{\bn}_{\bw,\bv} + \bar{\bn}_{\bv,\bw}}+\bar{\bn}_{\bv,\bw}\bar{\bw}^{\top}\frac{\bar{\bw}-\cos\p{\theta_{\bw,\bv}}\bar{\bv}}{\sin\p{\theta_{\bw,\bv}}}+\bar{\bn}_{\bw,\bv}\bar{\bv}^{\top}\frac{\bar{\bv}-\cos\p{\theta_{\bw,\bv}}\bar{\bw}}{\sin\p{\theta_{\bw,\bv}}}}\\
		=& \frac{1}{2\pi}\p{\p{\pi-\theta_{\bw,\bv}}\p{\bar{\bn}_{\bw,\bv} + \bar{\bn}_{\bv,\bw}}+\bar{\bn}_{\bv,\bw}\frac{1-\cos^2\p{\theta_{\bw,\bv}}}{\sin\p{\theta_{\bw,\bv}}}+\bar{\bn}_{\bw,\bv}\frac{1-\cos^2\p{\theta_{\bw,\bv}}}{\sin\p{\theta_{\bw,\bv}}}}\\
		=& \frac{1}{2\pi}\p{\p{\pi-\theta_{\bw,\bv}}\p{\bar{\bn}_{\bw,\bv} + \bar{\bn}_{\bv,\bw}}+\sin\p{\theta_{\bw,\bv}}\p{\bar{\bn}_{\bw,\bv} + \bar{\bn}_{\bv,\bw}}}\\
		=& \frac{1}{2\pi}\p{\pi-\theta_{\bw,\bv}+\sin\p{\theta_{\bw,\bv}}}\p{\bar{\bn}_{\bw,\bv} + \bar{\bn}_{\bv,\bw}}.
		\end{align*}
		Therefore $ \frac{1}{2\pi}\p{\pi-\theta_{\bw,\bv}+\sin\p{\theta_{\bw,\bv}}} $ is the largest eigenvalue of $ h_2 $.
	\end{itemize}
\end{proof}

\bibliographystyle{abbrvnat}
\bibliography{bib}

\begin{thebibliography}{27}
\providecommand{\natexlab}[1]{#1}
\providecommand{\url}[1]{\texttt{#1}}
\expandafter\ifx\csname urlstyle\endcsname\relax
  \providecommand{\doi}[1]{doi: #1}\else
  \providecommand{\doi}{doi: \begingroup \urlstyle{rm}\Url}\fi

\bibitem[Auer et~al.(1996)Auer, Herbster, and Warmuth]{auer1996exponentially}
P.~Auer, M.~Herbster, and M.~K. Warmuth.
\newblock Exponentially many local minima for single neurons.
\newblock In \emph{NIPS}, 1996.

\bibitem[Bhojanapalli et~al.(2016)Bhojanapalli, Neyshabur, and
  Srebro]{bhojanapalli2016global}
S.~Bhojanapalli, B.~Neyshabur, and N.~Srebro.
\newblock Global optimality of local search for low rank matrix recovery.
\newblock In \emph{Advances in Neural Information Processing Systems}, pages
  3873--3881, 2016.

\bibitem[Boob and Lan(2017)]{boob2017theoretical}
D.~Boob and G.~Lan.
\newblock Theoretical properties of the global optimizer of two layer neural
  network.
\newblock \emph{arXiv preprint arXiv:1710.11241}, 2017.

\bibitem[Brutzkus and Globerson(2017)]{brutzkus2017globally}
A.~Brutzkus and A.~Globerson.
\newblock Globally optimal gradient descent for a convnet with gaussian inputs.
\newblock \emph{arXiv preprint arXiv:1702.07966}, 2017.

\bibitem[Cho and Saul(2009)]{cho2009kernel}
Y.~Cho and L.~K. Saul.
\newblock Kernel methods for deep learning.
\newblock In \emph{Advances in neural information processing systems}, pages
  342--350, 2009.

\bibitem[Du et~al.(2017)Du, Lee, Tian, Poczos, and Singh]{du2017gradient}
S.~S. Du, J.~D. Lee, Y.~Tian, B.~Poczos, and A.~Singh.
\newblock Gradient descent learns one-hidden-layer cnn: Don't be afraid of
  spurious local minima.
\newblock \emph{arXiv preprint arXiv:1712.00779}, 2017.

\bibitem[Feizi et~al.(2017)Feizi, Javadi, Zhang, and Tse]{feizi2017porcupine}
S.~Feizi, H.~Javadi, J.~Zhang, and D.~Tse.
\newblock Porcupine neural networks:(almost) all local optima are global.
\newblock \emph{arXiv preprint arXiv:1710.02196}, 2017.

\bibitem[Ge et~al.(2015)Ge, Huang, Jin, and Yuan]{ge2015escaping}
R.~Ge, F.~Huang, C.~Jin, and Y.~Yuan.
\newblock Escaping from saddle points—online stochastic gradient for tensor
  decomposition.
\newblock In \emph{Conference on Learning Theory}, pages 797--842, 2015.

\bibitem[Ge et~al.(2016)Ge, Lee, and Ma]{ge2016matrix}
R.~Ge, J.~D. Lee, and T.~Ma.
\newblock Matrix completion has no spurious local minimum.
\newblock In \emph{Advances in Neural Information Processing Systems}, pages
  2973--2981, 2016.

\bibitem[Ge et~al.(2017)Ge, Lee, and Ma]{ge2017learning}
R.~Ge, J.~D. Lee, and T.~Ma.
\newblock Learning one-hidden-layer neural networks with landscape design.
\newblock \emph{arXiv preprint arXiv:1711.00501}, 2017.

\bibitem[Haeffele and Vidal(2015)]{haeffele2015global}
B.~D. Haeffele and R.~Vidal.
\newblock Global optimality in tensor factorization, deep learning, and beyond.
\newblock \emph{arXiv preprint arXiv:1506.07540}, 2015.

\bibitem[Janzamin et~al.(2015)Janzamin, Sedghi, and
  Anandkumar]{janzamin2015beating}
M.~Janzamin, H.~Sedghi, and A.~Anandkumar.
\newblock Beating the perils of non-convexity: Guaranteed training of neural
  networks using tensor methods.
\newblock \emph{CoRR abs/1506.08473}, 2015.

\bibitem[Ledoux(2005)]{ledoux2005concentration}
M.~Ledoux.
\newblock \emph{The concentration of measure phenomenon}.
\newblock Number~89. American Mathematical Society, 2005.

\bibitem[Li and Yuan(2017)]{li2017convergence}
Y.~Li and Y.~Yuan.
\newblock Convergence analysis of two-layer neural networks with relu
  activation.
\newblock \emph{arXiv preprint arXiv:1705.09886}, 2017.

\bibitem[Livni et~al.(2014)Livni, Shalev-Shwartz, and
  Shamir]{livni2014computational}
R.~Livni, S.~Shalev-Shwartz, and O.~Shamir.
\newblock On the computational efficiency of training neural networks.
\newblock In \emph{NIPS}, pages 855--863, 2014.

\bibitem[Nguyen and Hein(2017)]{nguyen2017loss}
Q.~Nguyen and M.~Hein.
\newblock The loss surface of deep and wide neural networks.
\newblock \emph{arXiv preprint arXiv:1704.08045}, 2017.

\bibitem[Poston et~al.(1991)Poston, Lee, Choie, and Kwon]{poston1991local}
T.~Poston, C.-N. Lee, Y.~Choie, and Y.~Kwon.
\newblock Local minima and back propagation.
\newblock In \emph{Neural Networks, 1991., IJCNN-91-Seattle International Joint
  Conference on}, volume~2, pages 173--176. IEEE, 1991.

\bibitem[Safran and Shamir(2016)]{safran2016quality}
I.~Safran and O.~Shamir.
\newblock On the quality of the initial basin in overspecified neural networks.
\newblock In \emph{International Conference on Machine Learning}, pages
  774--782, 2016.

\bibitem[Shamir(2016)]{shamir2016distribution}
O.~Shamir.
\newblock Distribution-specific hardness of learning neural networks.
\newblock \emph{arXiv preprint arXiv:1609.01037}, 2016.

\bibitem[Soltanolkotabi et~al.(2017)Soltanolkotabi, Javanmard, and
  Lee]{soltanolkotabi2017theoretical}
M.~Soltanolkotabi, A.~Javanmard, and J.~D. Lee.
\newblock Theoretical insights into the optimization landscape of
  over-parameterized shallow neural networks.
\newblock \emph{arXiv preprint arXiv:1707.04926}, 2017.

\bibitem[Soudry and Carmon(2016)]{soudry2016no}
D.~Soudry and Y.~Carmon.
\newblock No bad local minima: Data independent training error guarantees for
  multilayer neural networks.
\newblock \emph{arXiv preprint arXiv:1605.08361}, 2016.

\bibitem[Sun et~al.(2015)Sun, Qu, and Wright]{sun2015nonconvex}
J.~Sun, Q.~Qu, and J.~Wright.
\newblock When are nonconvex problems not scary?
\newblock \emph{arXiv preprint arXiv:1510.06096}, 2015.

\bibitem[Swirszcz et~al.(2016)Swirszcz, Czarnecki, and
  Pascanu]{swirszcz2016local}
G.~Swirszcz, W.~M. Czarnecki, and R.~Pascanu.
\newblock Local minima in training of deep networks.
\newblock \emph{arXiv preprint arXiv:1611.06310}, 2016.

\bibitem[Tian(2017)]{tian2017analytical}
Y.~Tian.
\newblock An analytical formula of population gradient for two-layered relu
  network and its applications in convergence and critical point analysis.
\newblock \emph{arXiv preprint arXiv:1703.00560}, 2017.

\bibitem[Zhang et~al.(2016)Zhang, Bengio, Hardt, Recht, and
  Vinyals]{zhang2016understanding}
C.~Zhang, S.~Bengio, M.~Hardt, B.~Recht, and O.~Vinyals.
\newblock Understanding deep learning requires rethinking generalization.
\newblock \emph{arXiv preprint arXiv:1611.03530}, 2016.

\bibitem[Zhang et~al.(2017)Zhang, Panigrahy, Sachdeva, and
  Rahimi]{zhang2017electron}
Q.~Zhang, R.~Panigrahy, S.~Sachdeva, and A.~Rahimi.
\newblock Electron-proton dynamics in deep learning.
\newblock \emph{arXiv preprint arXiv:1702.00458}, 2017.

\bibitem[Zhong et~al.(2017)Zhong, Song, Jain, Bartlett, and
  Dhillon]{zhong2017recovery}
K.~Zhong, Z.~Song, P.~Jain, P.~L. Bartlett, and I.~S. Dhillon.
\newblock Recovery guarantees for one-hidden-layer neural networks.
\newblock \emph{arXiv preprint arXiv:1706.03175}, 2017.

\end{thebibliography}

\end{document}